\documentclass[10pt,conference]{IEEEtran}

\usepackage{times}
 \pdfoutput=1
 \usepackage[pagebackref=true,breaklinks=true,letterpaper=true,colorlinks,bookmarks=false]{hyperref}
\usepackage{url}
\usepackage{epsfig}
\usepackage{epstopdf}
\usepackage{graphicx}
\usepackage{amsmath}
\usepackage{amssymb}
\usepackage{ mathrsfs }
\usepackage{color}
\usepackage{algpseudocode}
\usepackage{algorithm}
\usepackage{placeins}
\usepackage{rotating}
\usepackage{capt-of,etoolbox}

\usepackage{ulem} 

\newcommand{\argmin}{\operatornamewithlimits{argmin}}
\DeclareMathOperator{\tr}{tr}
\DeclareMathOperator{\prox}{prox}

\DeclareMathOperator{\sign}{sgn}
\DeclareMathOperator{\Larg}{\mathcal{L}}
\renewcommand{\Re}{\mathcal{R}}

\newtheorem{thm}{Theorem}
\newtheorem{defn}[thm]{Definition}

\ifCLASSINFOpdf
 
\else

\fi

\begin{document}

\title{Fast Robust PCA on Graphs}


\author{Nauman Shahid$^{*}$, Nathanael Perraudin, Vassilis Kalofolias, Gilles Puy$^\dagger$, Pierre Vandergheynst\\
Email: \{nauman.shahid, nathanael.perraudin, vassilis.kalofolias, pierre.vandergheynst\}@epfl.ch, $\dagger$ gilles.puy@inria.fr \\
Signal Processing Laboratory 2 (LTS2), EPFL STI IEL, Lausanne, CH-1015, Switzerland. \\
$\dagger$ INRIA Rennes - Bretagne Atlantique, Campus de Beaulieu, FR-35042 Rennes Cedex, France
}

\makeatletter
\let\@oldmaketitle\@maketitle
\renewcommand{\@maketitle}{\@oldmaketitle
  \centering \includegraphics[width=1.0\linewidth]{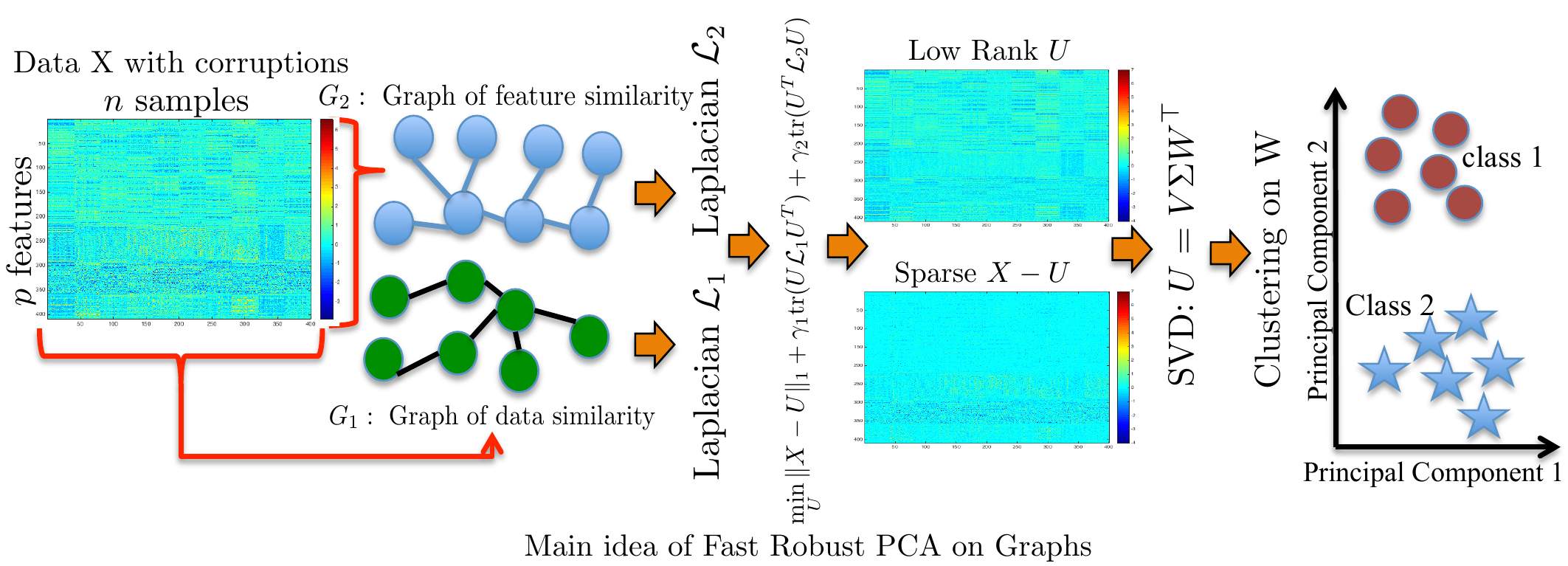}\bigskip}
\makeatother

\maketitle

\begin{abstract}
 Mining useful clusters from high dimensional data has received significant attention of the computer vision and pattern recognition community in the recent years. Linear and non-linear dimensionality reduction has played an important role to overcome the curse of dimensionality. However, often such methods are accompanied with three different problems: high computational complexity (usually associated with the nuclear norm minimization), non-convexity (for matrix factorization methods) and susceptibility to gross corruptions in the data. In this paper we propose a principal component analysis (PCA) based solution that overcomes these three issues and approximates a low-rank recovery method for high dimensional datasets. We target the low-rank recovery by enforcing two types of graph smoothness assumptions, one on the data samples and the other on the features by designing a convex optimization problem. The resulting algorithm is fast, efficient and scalable for huge datasets with $\mathcal{O}(n \log(n))$ computational complexity in the number of data samples. It is also robust to gross corruptions in the dataset as well as to the model parameters. Clustering experiments on $7$ benchmark datasets with different types of corruptions and background separation experiments on $3$ video datasets show that our proposed model outperforms $10$ state-of-the-art dimensionality reduction models. Our theoretical analysis proves that the proposed model is able to recover approximate low-rank representations with a bounded error for clusterable data.
\end{abstract}
\IEEEpeerreviewmaketitle


\section{Introduction}

In the modern era of data explosion, many problems in signal and image processing, machine learning and pattern recognition require dealing with very high dimensional datasets, such as images, videos and web content. The data mining community often strives to reveal natural associations or hidden structures in the data. Over the past couple of decades  matrix factorization has been adopted as one of the key methods in this context. Given a data matrix $X \in \mathbb{R}^{p \times n}$ with $n$ $p$-dimensional data vectors, the matrix factorization can be stated as determining $V\in \mathbb{R}^{p\times c}$ and $W \in \mathbb{R}^{c\times n}$ such that $X\approx VW$ under different constraints on $V$ and $W$. 

How can matrix factorization extract structures in the data? The answer to this question lies in the intrinsic association of linear dimensionality reduction with matrix factorization. 
Consider a set of gray-scale images of the same object captured under fixed lighting conditions with a moving camera, or a set of hand-written digits with different rotations. Given that the image has $m^{2}$ pixels, each such data sample is represented by a vector in $\mathbb{R}^{m^{2}}$. However, the intrinsic dimensionality of the space of all images of the same object captured with small perturbations is much lower than $m^{2}$. Thus, dimensionality reduction comes into play. Depending on the application and the type of data, one can either use a single linear subspace to approximate the data of different classes using the standard Principal Component Analysis (PCA) \cite{abdi2010principal}, a union of low dimensional subspaces where each class belongs to a different subspace (LRR and SSC) \cite{elhamifar2013sparse,liu2013robust,vidal2014low,shahid2015robust},  or a positive subspace to extract a positive low-rank representation of the data (NMF) \cite{lee1999learning}. The clustering or community detection can then be performed on the retrieved data representation in the low dimensional space. Not surprisingly, all the above mentioned problems can be stated in the standard matrix factorization manner as shown in the models 1 to 3 of Fig.~\ref{fig:MF}.  Alternatively, the clustering quality for non-linearly separable datasets can be improved by using non-linear dimensionality reduction tools such as Laplacian Eigenmaps \cite{belkin2003laplacian} or Kernel PCA \cite{scholkopf1997kernel}. 

\begin{figure*}[htbp]
    \centering
        \centering
        \includegraphics[width=1.0\textwidth]{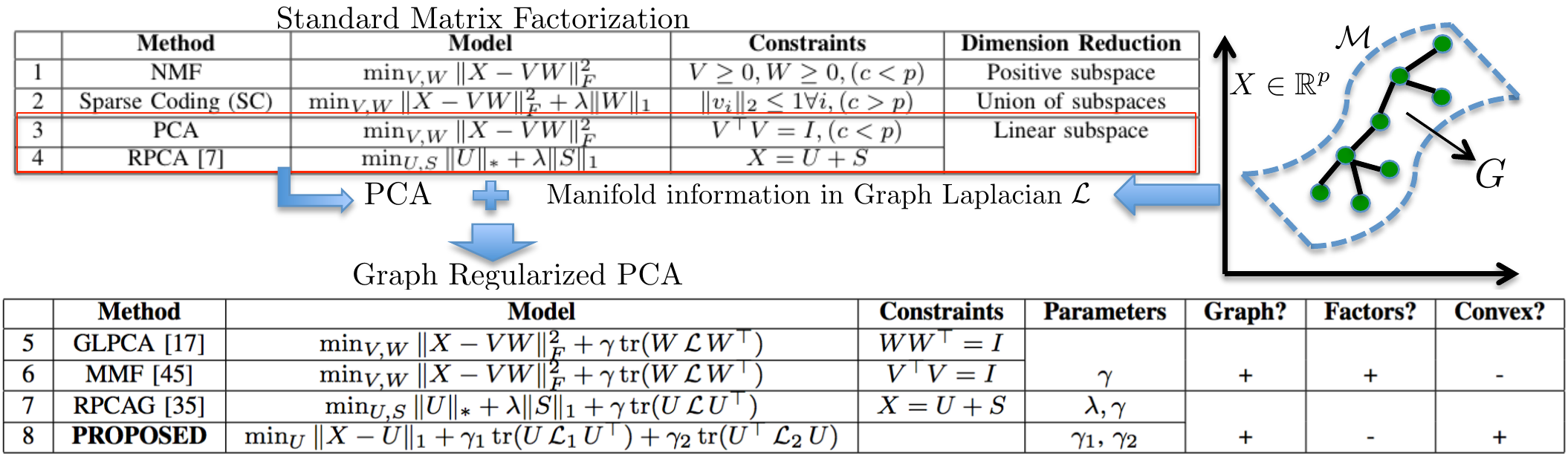}
         \caption{A summary of the matrix factorization methods with and without graph regularization. $X \in \mathbb{R}^{p\times n}$ is the matrix of $n$ $p$-dimensional data vectors, $V \in \mathbb{R}^{p\times c}$ and $W \in \mathbb{R}^{c\times n}$ are the learned factors. $U \in \mathbb{R}^{p\times n}$ is the low-rank  matrix and $S \in \mathbb{R}^{p\times n}$ is the sparse matrix. $\|\cdot\|_{F}$, $\|\cdot\|_{*}$ and $\|\cdot\|_{1}$ denote the Frobenius, nuclear and $\ell_1$ matrix norms respectively. The data manifold $\mathcal{M}$ information can be leveraged in the form of a discrete graph $G$ using the graph Laplacian $\mathcal{L} \in \mathbb{R}^{n \times n}$ resulting in various Graph Regularized PCA models.  }
        \label{fig:MF}
    \end{figure*}


In many cases low dimensional data follows some additional structure. Knowledge of such structure is beneficial, as we can use it to enhance the representativity of our models by adding structured priors \cite{jenatton2009structured, mairal2014sparse,wang2013nonnegative}. A nowadays standard way to represent pairwise affinity between objects is by using graphs. The introduction of graph-based priors to enhance matrix factorization models has recently brought them back to the highest attention of the data mining community. Representation of a signal on a graph is well motivated by the emerging field of signal processing on graphs, based on notions of spectral graph theory \cite{shuman2013emerging}. The underlying assumption is that high-dimensional data samples lie on or close to a smooth low-dimensional manifold. Interestingly, the underlying manifold can be represented by its discrete proxy, i.e. a graph. Let $G = \mathcal{(V,E)}$ be a graph between the samples of $X$, where $\mathcal{E}$ is the set of edges and $\mathcal{V}$ is the set of vertices (data samples). Let $A$ be the symmetric matrix that encodes the weighted adjacency information between the samples of $X$ and $D$ is the diagonal degree matrix with $D_{ii} = \sum_j A_{ij}$. Then the normalized graph Laplacian $\Larg$ that characterizes the graph $G$ is defined as $\Larg = D^{-1/2}(D-A)D^{-1/2}$. Exploiting the manifold information in the form of a graph can be seen as a method of incorporating local proximity information of the data samples into the dimensionality reduction framework, that can enhance the clustering quality in the low-dimensional space.

\subsection{Focus of this work}
In this paper, we focus on the application of PCA to clustering, projecting the data on a single linear subspace. We first describe PCA and its related models and then elaborate on how the data manifold information in the form of a graph can be used to enhance standard PCA. Finally, we present a novel, convex, fast and scalable  method for PCA that recovers the low-rank representation via two graph structures.  Our theoretical analysis proves that the proposed model is able to recover approximate low-rank representations with a bounded error for clusterable data, where the number of clusters is equal to the rank. Many real world datasets can be assumed to satisfy this assumption. For example, the USPS dataset which consists of ten digits.  We call such  data matrices as low-rank matrices on graphs. The clustering on these dataset can be done by recovering a clean low-rank representation.

\subsection{PCA and Related Work}
For a dataset $X\in \mathbb{R}^{p\times n}$ with $n$ $p$-dimensional data vectors, standard PCA learns the projections or principal components $W\in \mathbb{R}^{c\times n}$ of $X$ on a $c$-dimensional orthonormal basis $V\in \mathbb{R}^{p\times c}$, where $c < p$ by solving model 3 in Fig.~\ref{fig:MF}. Though non-convex, this problem has a global minimum that can be computed using Singular Value Decomposition (SVD), giving a unique low-rank representation $U = VW$.


A main drawback of PCA is its sensitivity to heavy-tailed noise due to the Frobenius norm in the objective function. Thus, a few strong corruptions  can result in erratic principal components.  Robust PCA (RPCA) proposed by Candes et al. \cite{candes2011robust} overcomes this problem by recovering the clean low-rank representation $U$ from grossly corrupted $X$ by solving model 4 in Fig.~\ref{fig:MF}. Here $S$ represents the sparse matrix containing the errors and $\|U\|_{*}$ denotes the nuclear norm of $U$, the tightest convex relaxation of $\text{rank}(U)$.

Recently, many works related to low-rank or sparse representation recovery have been proposed to incorporate the data manifold information in the form of a discrete graph into the dimensionality reduction framework \cite{jiang2013graph, zhang2013low, gao2013laplacian, cai2011graph, tao2014low,jin2014multiple,jin2014low,peng2015enhanced,du2015sparse}.  In fact, for PCA, this can be considered as a method of exploiting the local smoothness information in order to improve clustering quality. The graph smoothness of the principal components $W$ using the graph Laplacian $\Larg$ has been exploited in various works that explicitly learn $W$ and the basis $V$. We  refer to such models as \textit{factorized models}. In this context Graph Laplacian PCA (GLPCA) was proposed in \cite{jiang2013graph} (model 5 in Fig.~\ref{fig:MF}) and Manifold Regularized Matrix Factorization (MMF) in \cite{zhang2013low} (model 6 in Fig.~\ref{fig:MF}). Note that the orthonormality constraint in this model is on $V$, instead of the principal components $W$.
Later on, the authors of \cite{shahid2015robust} have generalized robust PCA by incorporating the graph smoothness (model 7 in Fig.~\ref{fig:MF}) term directly on the low-rank matrix instead of principal components. They call it Robust PCA on Graphs (RPCAG).

Models 4 to 8 can be used for clustering in the low dimensional space. However, each of them comes with its own weaknesses. GLPCA \cite{jiang2013graph} and MMF  \cite{zhang2013low} improve upon the classical PCA by incorporating graph smoothness but they are non-convex and susceptible to data corruptions. Moreover, the rank $c$ of the subspace has to be specified upfront. RPCAG \cite{shahid2015robust} is convex and builds on the robustness property of RPCA \cite{candes2011robust} by incorporating the graph smoothness directly on the low-rank matrix and improves both the clustering and low-rank recovery properties of PCA. However, it uses the nuclear norm relaxation that involves an expensive SVD step in every iteration of the algorithm.  Although fast methods for the SVD have been proposed, based on randomization  \cite{witten2013randomized,lucas2014parallel,oh2015fast}, Frobenius norm based representations \cite{zhang2014flrr,peng2015connections} or structured RPCA \cite{ayazoglu2012fast}, its use in each iteration makes it hard to scale to large datasets.

\subsection{Our Contributions}
In this paper we propose a fast, scalable,  robust and convex clustering and low-rank recovery method for potentially corrupted low-rank signals. Our contributions are:
\begin{enumerate}
\item We propose an approximate low-rank recovery method  for corrupted data by utilizing only the graph smoothness assumptions both between the samples and between the features.
\item Our theoretical analysis proves that the proposed model is able to recover approximate low-rank representations with a bounded error for clusterable data, where the number of clusters is equal to the rank. We call such a data matrix a \textit{low-rank matrix on the graph}.
\item Our model is convex and although non-smooth it can be solved efficiently, that is in linear time in the number of samples, with a few iterations of the well-known FISTA algorithm. The construction of the two graphs costs $\mathcal{O}(n\log n)$ time, where $n$ is the number of data samples.
\item The resulting algorithm is highly parallelizable and scalable for large datasets since it requires only the multiplication of two sparse matrices with full vectors and elementwise soft-thresholding operations.
\item Our extensive experimentation shows that the recovered close-to-low-rank matrix is a good approximation of the low-rank matrix obtained by solving the expensive state-of-the-art method \cite{shahid2015robust} which uses the much more expensive nuclear norm. This is observed even in the presence of gross corruptions in the data. 
\end{enumerate}

\subsection{Connections and differences with the state-of-the-art}
The idea of using two graph regularization terms has previously appeared in the work of matrix completion \cite{kalofolias2014matrix}, co-clustering \cite{gu2009co}, NMF \cite{shang2012graph}, \cite{BenziKBV16arxiv} and more recently in the context of low-rank representation \cite{yin2015dual}. However, to the best of our knowledge all these models aim to improve the clustering quality of the data in the low-dimensional space.  The co-clustering \& NMF based models which use such a scheme \cite{gu2009co}, \cite{shang2012graph} suffer from non-convexity and the works of  \cite{kalofolias2014matrix} and \cite{yin2015dual} use a nuclear-norm formulation which is computationally expensive and not scalable for big datasets. Our proposed method is different from these models in the following sense:
\begin{itemize}
\item We do not target an improvement in the low-rank representation via graphs. Our method aims to solely recover an approximate low-rank matrix with dual-graph regularization only. The underlying motivation is that one can obtain a good enough low-rank representation without using expensive nuclear norm or non-convex matrix factorization. Note that the NMF-based method   \cite{shang2012graph} targets the smoothness of factors of the low-rank while the co-clustering \cite{gu2009co} focuses on the smoothness of the labels. \textit{Our method, on the other hand, targets directly the recovery of the low-rank matrix, and not the one of the factors or labels}.
\item  We introduce the concept of low-rank matrices on graphs and provide a theoretical justification for the success of our model. The  use of PCA as a scalable and efficient clustering method using dual graph regularization has surfaced for the very first time in this paper.  
\end{itemize}

A summary of the notations used in this paper is presented in Tab.~\ref{tab:notations}.  We first introduce our proposed formulation and its optimization solution in Sections~\ref{sec:proposed_model} \&~\ref{sec:optimization} and then develop a sound motivation of the model in Section~\ref{sec:motivation}. 

 \begin{table}[htbp]
\footnotesize
\caption{A summary of notations used in this work}
\centering
\resizebox{0.5\textwidth}{!}{\begin{tabular}[t]{| c | c | } \hline
\textbf{Notation}   & \textbf{Terminology} \\\hline
$\|\cdot \|_{F}$  & matrix frobenius norm \\\hline
$\|\cdot\|_{1}$ & matrix  $\ell_1$ norm \\\hline
$n$      & number of data samples \\\hline
$p$      & number of features / pixels \\\hline
$c$      & dimension of the subspace \\\hline
$k$  & number of classes in the data set \\\hline
$X \in \mathbb{R}^{p\times n}$  & data matrix \\\hline
$U \in \mathbb{R}^{p\times n}$  & low-rank noiseless approximation of $X$ \\\hline
$U= V\Sigma W^{\top}$  & SVD of the low-rank matrix $U$ \\\hline
$V \in \mathbb{R}^{p\times c}$   & left singular vectors of $U$ / principal directions of $U$ \\\hline
$\Sigma$   & singular values of $U$ \\\hline
$W \in \mathbb{R}^{n\times c}$ & right singular vectors of $U$ / principal components of $U$ \\\hline
$A \in \mathbb{R}^{n\times n}$ or $\mathbb{R}^{p\times p}$  & adjacency matrix between samples / features of $X$ \\\hline
$D = diag(\sum_{j}A_{ij}) \forall i$      & diagonal degree matrix \\\hline
$\sigma$   & smoothing parameter of the Gaussian kernel \\\hline
$G_{1}$  & graph between the samples of $X$ \\\hline
$G_2$  & graph between the features of $X$ \\\hline
$\mathcal{(V,E)}$ & set of vertices, edges for graph \\\hline
$\gamma_1$  & penalty for $G_1$ Tikhonov regularization term \\\hline
$\gamma_2$  & penalty for $G_2$ Tikhonov regularization term \\\hline
$K$   & nearest neighbors for the construction of graphs \\\hline
$\Larg_1 \in \mathbb{R}^{n\times n}$ & Laplacian for graph $G_1$ \\\hline
$\Larg_2 \in \mathbb{R}^{p\times p}$ &  Laplacian for graph $G_2$ \\\hline
$\Larg_{1} = Q\Lambda Q^{\top}$ & eigenvalue decomposition of $\Larg_1$ \\\hline
$\Larg_2 = P\Omega P^{\top}$  & eigenvalue decomposition of $\Larg_2$ \\\hline
\end{tabular}}
\label{tab:notations}
\end{table}

\section{Fast Robust PCA on Graphs (FRPCAG)}\label{sec:proposed_model}    
Let $\Larg_{1}\in \mathbb{R}^{n\times n}$ be the graph Laplacian of the graph $G_{1}$ connecting the different samples of $X$ (columns of $X$) and $\Larg_{2} \in \mathbb{R}^{p\times p}$ the Laplacian of graph $G_{2}$ that connects the features of $X$ (rows of $X$). The  construction of these two graphs is described in Section~\ref{sec:graphs}. We denote by $U\in \mathbb{R}^{p\times n}$ the low-rank noiseless matrix that needs to be recovered from the measures $X$, then our proposed model can be written as: 
\begin{align}\label{eq:proposed}
\min_{U} \|X-U\|_{1} + \gamma_{1}\tr(U\Larg_{1} U^\top ) + \gamma_{2}\tr(U^\top \Larg_{2} U).
\end{align}
  This problem can be reformulated in the equivalent split form 
\begin{align}\label{eq:proposed1}
& \min_{U,S} \|S\|_{1} + \gamma_{1}\tr(U\Larg_{1} U^\top ) + \gamma_{2}\tr(U^\top \Larg_{2} U), \\
& \text{s.t.} ~ X = U + S, \nonumber
\end{align}
where $S$ models the sparse outliers in the data $X$. The $\| \cdot \|_1$ denotes the element-wise $L_1$ norm of a matrix.
Model ~\eqref{eq:proposed1} has close connections with the RPCAG \cite{shahid2015robust}. In fact the nuclear norm term in RPCAG has been replaced by another graph Tikhonov term. The two graph regularization terms help in retrieving an approximate low-rank representation $U$ by encoding graph smoothness assumptions on $U$ without using the expensive nuclear norm of RPCAG, therefore we call it Fast Robust PCA on Graphs (FRPCAG). The main idea of our work is summarized in the fig. of the first page of this paper.

\subsection{Optimization Solution}\label{sec:optimization}
We use the Fast Iterative Soft Thresholding Algorithm (FISTA) \cite{beck2009fast} to solve problem~\eqref{eq:proposed}. Let $g: \mathbb{R^{N}}\rightarrow \mathbb{R}$ be a convex, differentiable function with a $\beta$-Lipschitz continuous gradient $\nabla g$ and $h: \mathbb{R^{N}}\rightarrow \mathbb{R}$ a convex function with a proximity operator $\prox_{h}:\mathbb{R}^N\rightarrow\mathbb{R}^N$ defined as:
\begin{equation*}
 \prox_{\lambda h}(y) = \argmin_x \frac{1}{2} \|x-y\|_2^2 + \lambda  h(x) .
\end{equation*}
Our goal is to minimize the sum $g(x)+h(x)$, which is done efficiently with proximal splitting methods. More information about proximal operators and splitting methods for non-smooth convex optimization can be found in \cite{combettes2011proximal}.
For model~\eqref{eq:proposed}, $g(U) = \gamma_{1}\tr(U\Larg_{1}U^\top ) + \gamma_{2}\tr(U^\top \Larg_{2}U)$ and $h(U) = \|X-U\|_{1}$. The gradient of $g$ becomes
\begin{equation}\label{eq:grad}
\nabla_g(U) =  2( \gamma_{1} U\Larg_{1}  + \gamma_{2} \Larg_{2}U).
\end{equation}
We define an upper bound on the Lipschitz constant $\beta$ as $\beta \leq \beta' = 2\gamma_1 \|\Larg_1\|_2 + 2\gamma_2 \|\Larg_2\|_2$ where $\|\Larg \|_2$ is the  spectral norm (or maximum eigenvalue) of $\Larg$. Moreover, the proximal operator of the function $h$ is the $\ell_1$ soft-thresholding given by the elementwise operations (here $\circ$ is the Hadamard product)
\begin{equation}\label{eq:prox}
\prox_{\lambda h }(U) = X + \sign(U-X) \circ \max (|U-X|-\lambda ,0).
\end{equation}
The FISTA algorithm \cite{beck2009fast} can now be stated as Algorithm \ref{CHalgorithm},
\begin{algorithm}
\caption{FISTA for FRPCAG}
\label{CHalgorithm}
\begin{algorithmic}
\State INPUT: $Y_1 = X$, $U_0 = X$, $t_1 = 1$, $\epsilon > 0$
\For{ $j = 1,\dots J$ }
\State $U_{j} = \prox_{\lambda_{j}h}(Y_{j}-\lambda_{j}\nabla g(Y_{j}))$
\State $t_{j+1} = \frac{1+\sqrt{1+4t_j^2}}{2}$
\State $Y_{j+1} = U_j +\frac{t_j-1}{t_{j+1}} (U_j-U_{j-1})$
\If{$\|Y_{j+1} - Y_{j}\|_F^2 < \epsilon \| Y_{j}\|_F^2$}
\State BREAK
\EndIf
\EndFor
\State OUTPUT: $U_{j+1}$
\end{algorithmic}
\end{algorithm}
where $\lambda$ is the step size (we use $\lambda = \frac{1}{\beta '}$), $\epsilon$ the stopping tolerance and $J$ the maximum number of iterations. 

\section{Graphs Construction}\label{sec:graphs}
We use two types of graphs $G_{1}$ and $G_{2}$ in our proposed model. The graph $G_{1}$ is constructed between the data samples or the columns of the data matrix and the graph $G_{2}$ is constructed between the features or the rows of the data matrix.  The graphs are undirected and built using a standard and a fast K-nearest neighbor strategy.  The first step consists of searching the closest neighbours for all the samples using Euclidean distances.  We connect each $x_i$ to its $K$ nearest neighbors $x_j$, resulting in $|\mathcal{E}|$ number of connections. The K-nearest neighbors are non-symmetric.  The second step consists of computing the graph weight matrix $A$ as
\begin{equation*}
A_{ij} = \begin{cases}
\exp\Big(-\frac{ \|(x_i-x_j)\|^{2}_{2}}{\sigma^{2}}\Big) & \text{if $x_j$ is connected to $x_i$}\\
0 & \text{otherwise.}\\
\end{cases}
\end{equation*}
The parameter $\sigma$ can be set empirically as the average distance of the connected samples. Provided that this parameter is not big, it does not effect the final quality of our algorithm.
Finally, in the third step, the normalized graph Laplacian $\Larg = I - D^{-1/2}AD^{-1/2}$ is calculated, where $D$ is the diagonal degree matrix. This procedure has a complexity of $\mathcal{O}(n e)$ and each $A_{ij}$ can be computed in parallel. Our choice of normalized Laplacian is arbitrary and depends on the application under consideration. An advantage of using a normalized laplacian as compared to an unnormalized is that all the eigenvalues for the normalized laplacian lie between 0 and 2 for all the datasets. This eases the comparison of the spectra of the laplacians. The eigenvalues of the unnormalized laplacian can be unbounded and have different ranges for different datasets. Depending on the values of $n$ and $p$  the above computation can be done in two different ways.


\textbf{Strategy 1}: For small $n$, $p$ we can use the above strategy directly for both $G_1$ and $G_2$ even if the dataset is corrupted. Although, the computation of $A$ is $\mathcal{O}(n^{2})$, it should be noted that with sufficiently small $n$ and $p$, the graphs $G_1$ and $G_2$ can still be computed in the order of a few seconds. 

\textbf{Strategy 2}: For big or high dimensional datasets, i.e, large $n$ or large $p$ or both, we can use a similar strategy but the computations can be made  efficient ($\mathcal{O}(n\log n)$) using the FLANN library (Fast Library for Approximate Nearest Neighbors searches in high dimensional spaces) \cite{muja2014scalable}. However, the quality of the graphs constructed using this strategy is slightly lower as compared to strategy 1 due to the approximate nearest neighbor search method.  We describe the complexity of FLANN in detail in Section~\ref{sec:complexity_o}.


Thus for our work the overall quality of graphs can be divided into 3 types. 
\begin{itemize}
\item \textbf{Type A}: Good sample graph $G_1$ and good feature graph $G_2$, both constructed using strategy 1. This case corresponds to small $n$ and $p$.
\item \textbf{Type B}: Good sample graph $G_1$  using strategy 1 and noisy feature graph $G_2$  using strategy 2. This case corresponds to small $n$ but large $p$. 
\item \textbf{Type C}: Noisy sample graph $G_1$ and noisy feature graph $G_2$ both constructed using strategy 2 for large $n$ and $p$.
\end{itemize}
We report the performance of FRPCAG for these three combinations of graph types, thus the acronyms FRPCAG(A), FRPCAG(B) and FRPCAG(C). Although the graph quality is lower if FLANN is used for corrupted data, our experiments for MNIST dataset show that our proposed model attains better results than other state-of-the-art models even with low quality graphs.
 
%
%
%
%

 \vspace{-0.2cm}   
\section{Our Motivation: Low-rank matrix on graphs}\label{sec:motivation}
In this section we lay down the foundation and motivation of our method and take a  step towards a theoretical analysis of FRPCAG. We build the motivation behind FRPCAG with a simple convincing demonstration. We start by answering the question: \textit{Why do we need two graphs?} This discussion ultimately leads to the introduction of a new concept, the \textit{low-rank matrix on a graph}. The latter models clusterable data and facilitates our theoretical analysis. 

\vspace{-0.2cm}
\subsection{The graph of features provides a basis for data}\label{sec:m_gf}
Consider  a simple example  of the digit $3$ from the traditional USPS dataset. We vectorize all the images and form a data matrix $X$, whose columns consist of different samples of digit $3$ from the USPS dataset. In order to motivate the need of the graph of features we build the $10$ nearest neighbors graph (of features), i.e, a graph between the rows of $X$ using the FLANN strategy of Section \ref{sec:graphs}.  Fig.~\ref{fig:exp3} shows the eigenvectors of the Laplacian denoted by $P$. We observe that they have a $3$-like shape. In Fig.~\ref{fig:exp3}, we also plot the eigenvectors associated to the experimental covariance matrix
$C$\footnote{The experimental covariance matrix is computed as $C = \frac{\tilde{X}\tilde{X}^\top}{n}$, where $n$ is the number of samples and $\tilde{X} = X - \mu_X$ for $\mu_X = \frac{1}{nd}\sum_{i=1}^d \sum_{j=1}^n X_{ij}.$ This definition is motivated in \cite{2016arXiv160102522P}.}. We observe that both sets of eigenvectors are similar. This is confirmed by computing the following matrix:
\begin{equation} \label{eq:Fourier_cov_matrix}
\Gamma = P^\top C P
\end{equation}
In order to measure the level of alignment between the orthogonal basis $P$ and the one behind $C$, we use the following ratio:
\begin{equation}
s_r(\Gamma) = \left(\frac{\sum_\ell \Gamma_{\ell,\ell}^2}{\sum_{\ell_1}\sum_{\ell_2} \Gamma_{\ell_1,\ell_2}^2} \right)^{\frac{1}{2}} =\frac{\| \rm{diag}(\Gamma) \|_2}{\| \Gamma \|_F}.
\end{equation}
When the two bases are aligned, the covariance matrix $C$ and the graph Laplacian $L$ are simultaneously diagonalizable, giving a ratio equal to $1$. On the contrary, when the bases are not aligned, the ratio is close to $\frac{1}{p}$, where $p$ is the dimension of the dataset. Note that an alternative would be to compute directly the inner product between $P$ and the eigenvectors of $C$. However, using $\Gamma$ we implicitly weight the eigenvectors of $C$ according to their importance given by their corresponding eigenvalues. 

In the special case of the digit $3$, we obtain a ratio $s_r(\Gamma_3) = 0.97$, meaning that the main covariance eigenvectors are well aligned to the graph eigenvectors. Fig.~\ref{fig:exp3} shows a few eigenvectors of both sets and the matrix $\Gamma$. 
\begin{figure}[htb!]
\begin{center}
\includegraphics[width=0.225\textwidth]{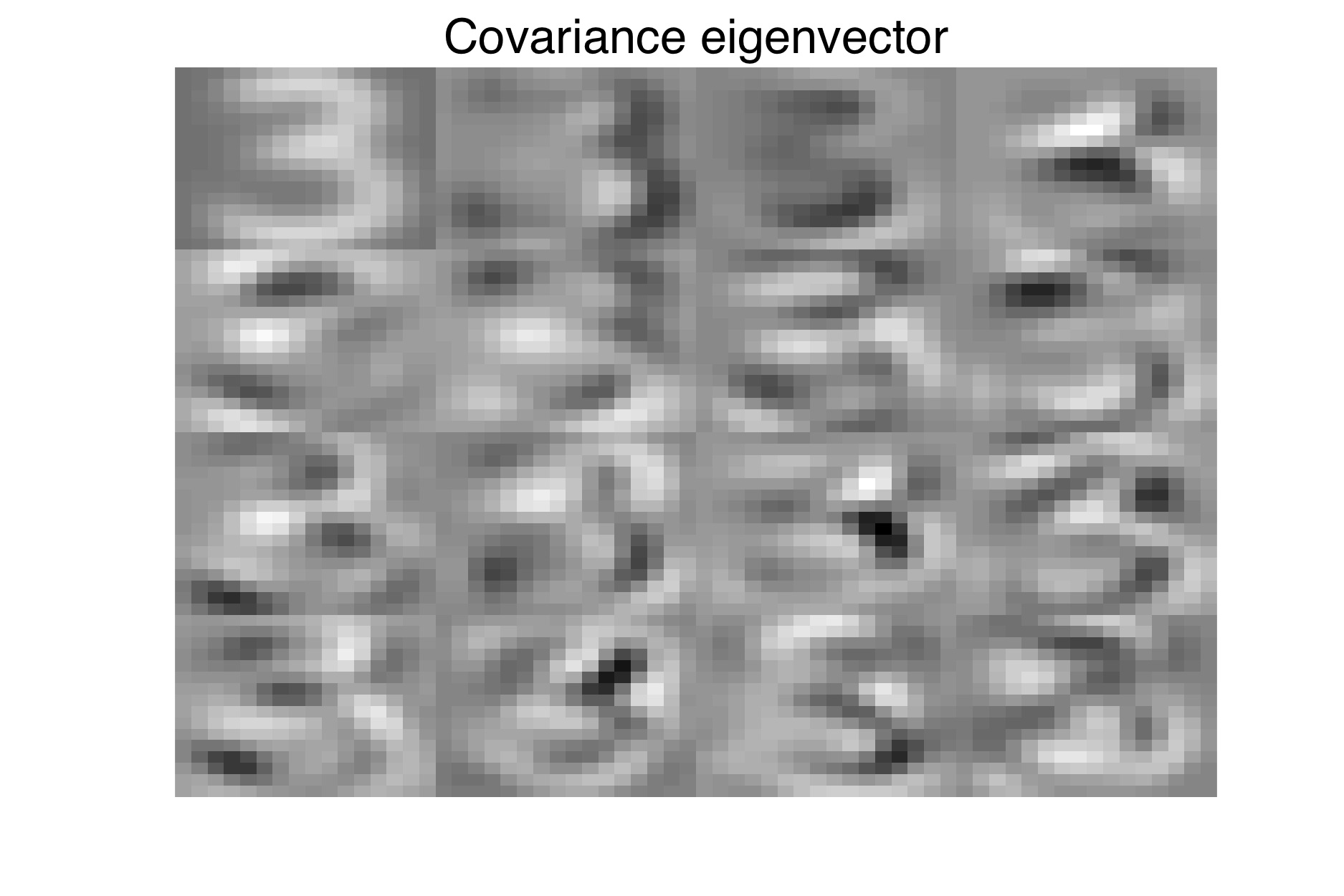} 
\includegraphics[width=0.225\textwidth]{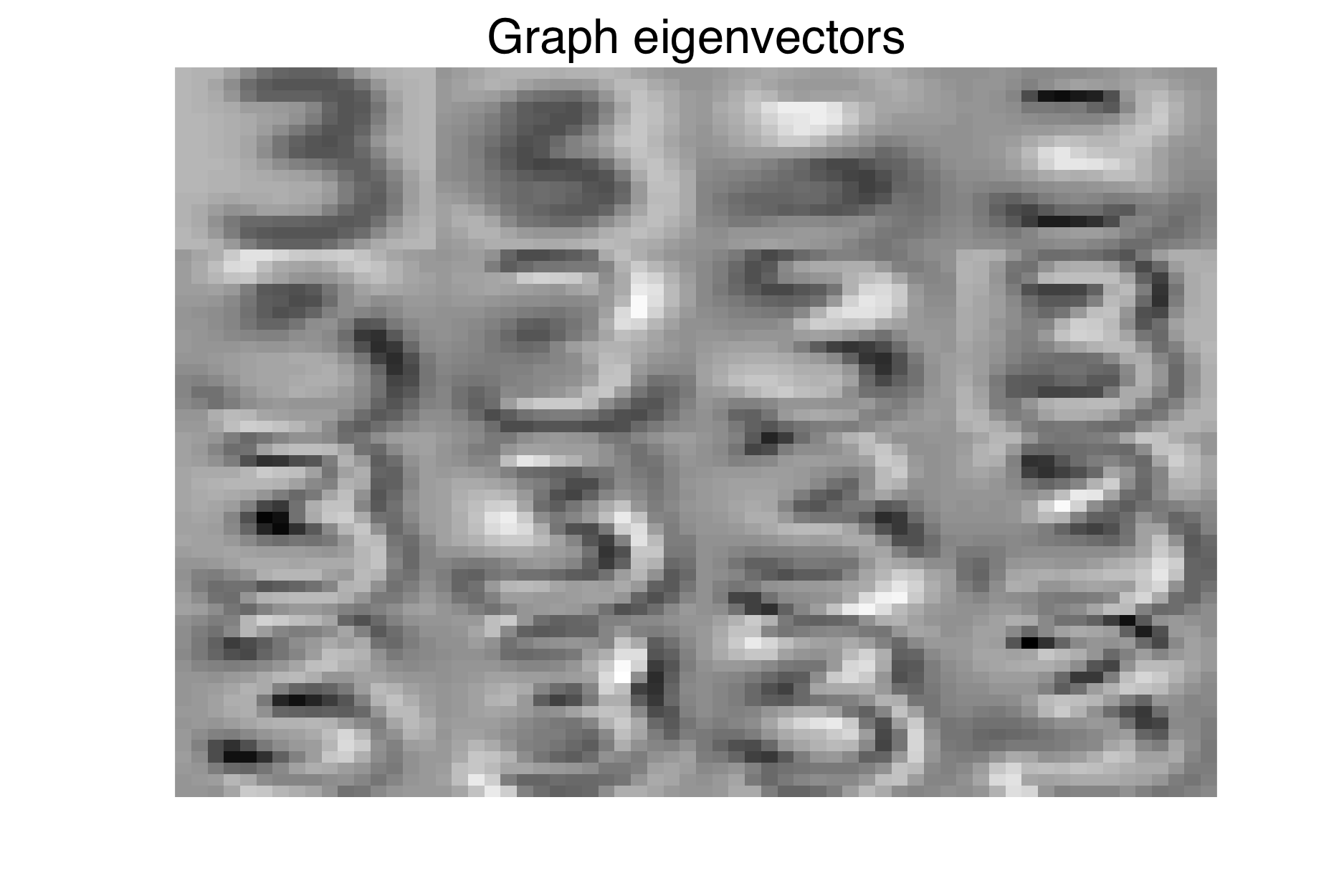} \\
\includegraphics[width=0.225\textwidth]{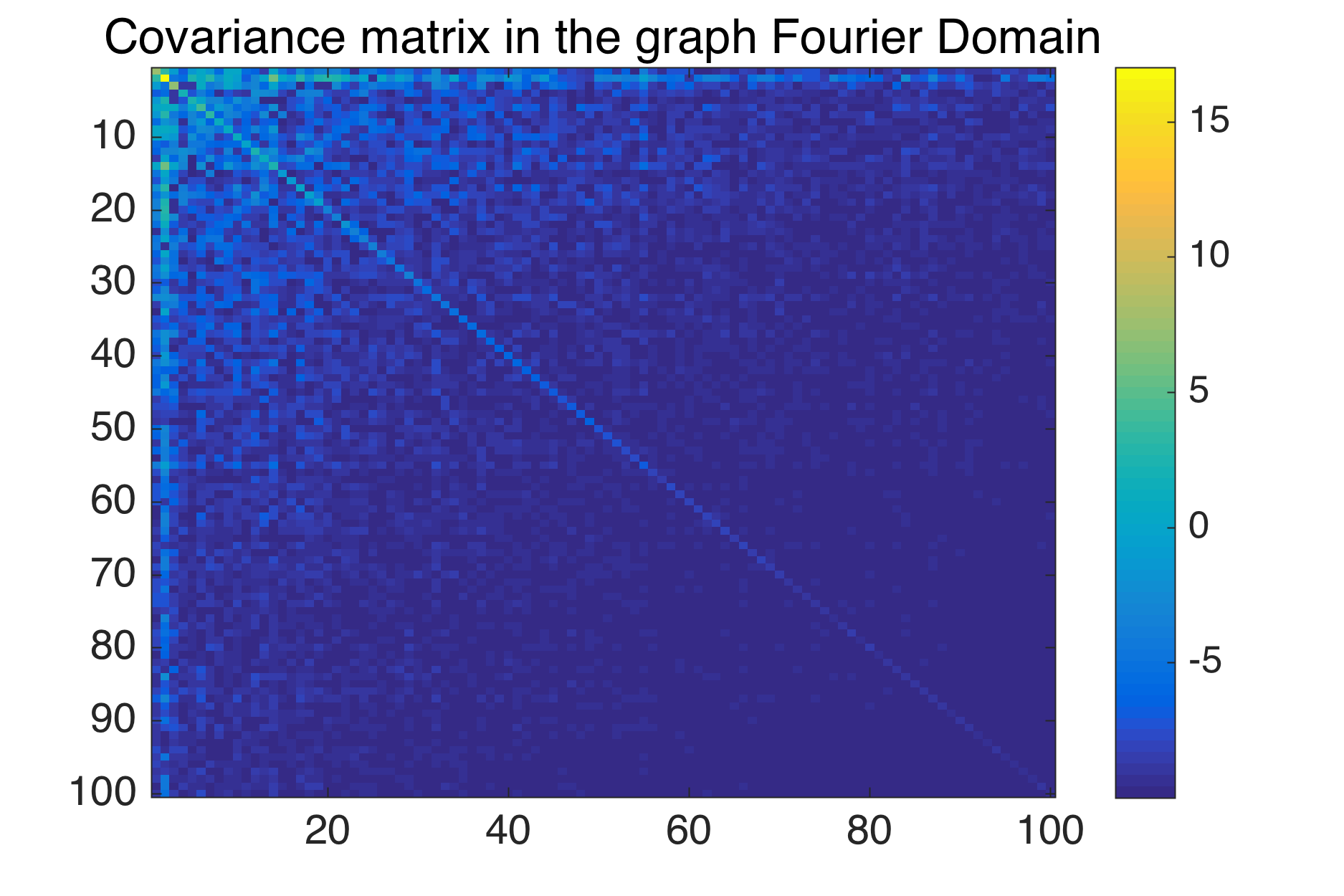}
\end{center}
\caption{Studying the number $3$ of USPS. Left: Covariance eigenvectors associated with the $16$ highest eigenvalues. Right: Laplacian eigenvectors associated to the $16$ smallest non-zero eigenvalues. Because of stationarity, Laplacian eigenvectors are similar to the covariance eigenvectors. Bottom: $\Gamma_3 = P^\top C_3 P$ in dB. Note the diagonal shape of the matrix implying that $P$ is aligned with the eigenvectors of $C$.}
\label{fig:exp3}
\end{figure}
This effect has been studied in \cite{2016arXiv160102522P} where the definition of stationary signals on graphs is proposed. A similar idea is also the motivation of the Laplacianfaces algorithm~\cite{he2005face}. A closer look at the bottom figure of Fig.~\ref{fig:exp3} shows that most of the energy in the diagonal is concentrated in the first few entries. This shows that the first few eigenvectors of the Laplacian are more aligned with the eigenvectors of the covariance matrix. This phenomena implies that the digit $3$ of the USPS dataset is low-rank, i.e, only the first few eigenvectors (corresponding to the low eigenvalues) are enough to serve as the features for this dataset.

Of course, FRPCAG also acts on the full dataset. Let us analyze how the graph eigenvectors evolve when all digits are taken into account. Fig.~\ref{fig:exp4} shows the Laplacian and covariance eigenvectors for the full USPS dataset. Again we observe some alignment: $s_r(\Gamma) = 0.82$. 
\begin{figure}[htb!]
\begin{center}
\includegraphics[width=0.225\textwidth]{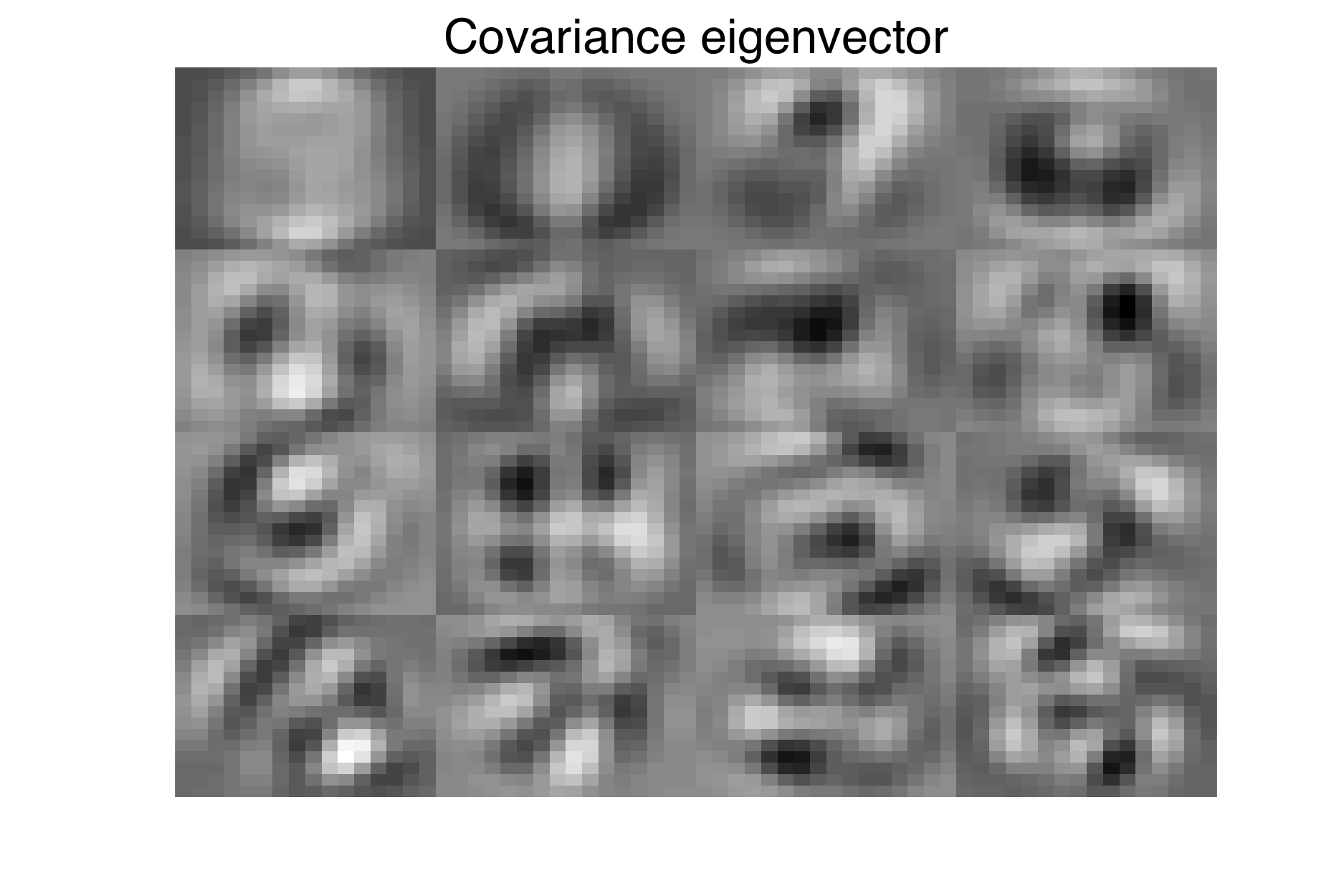} 
\includegraphics[width=0.225\textwidth]{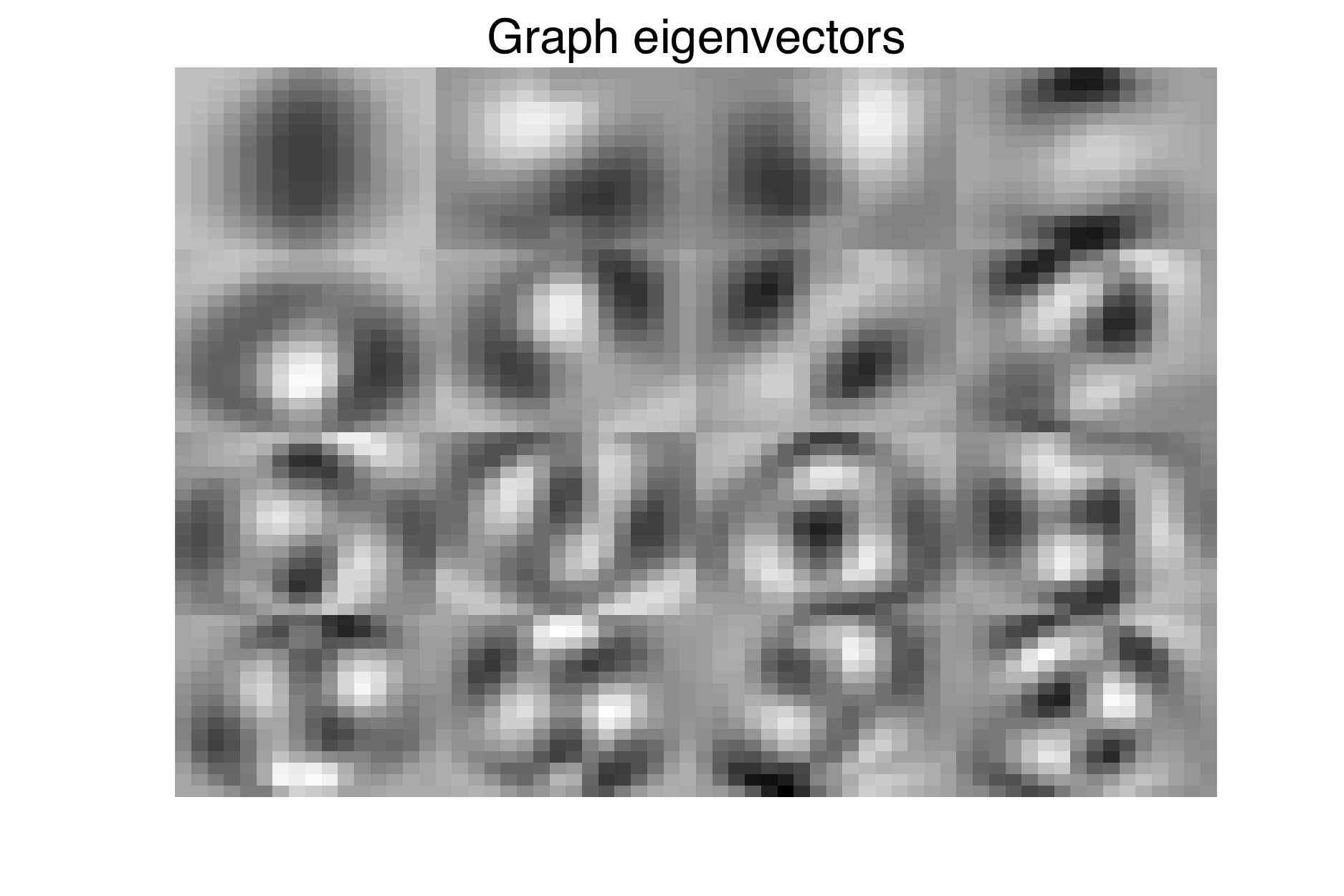} \\
\includegraphics[width=0.225\textwidth]{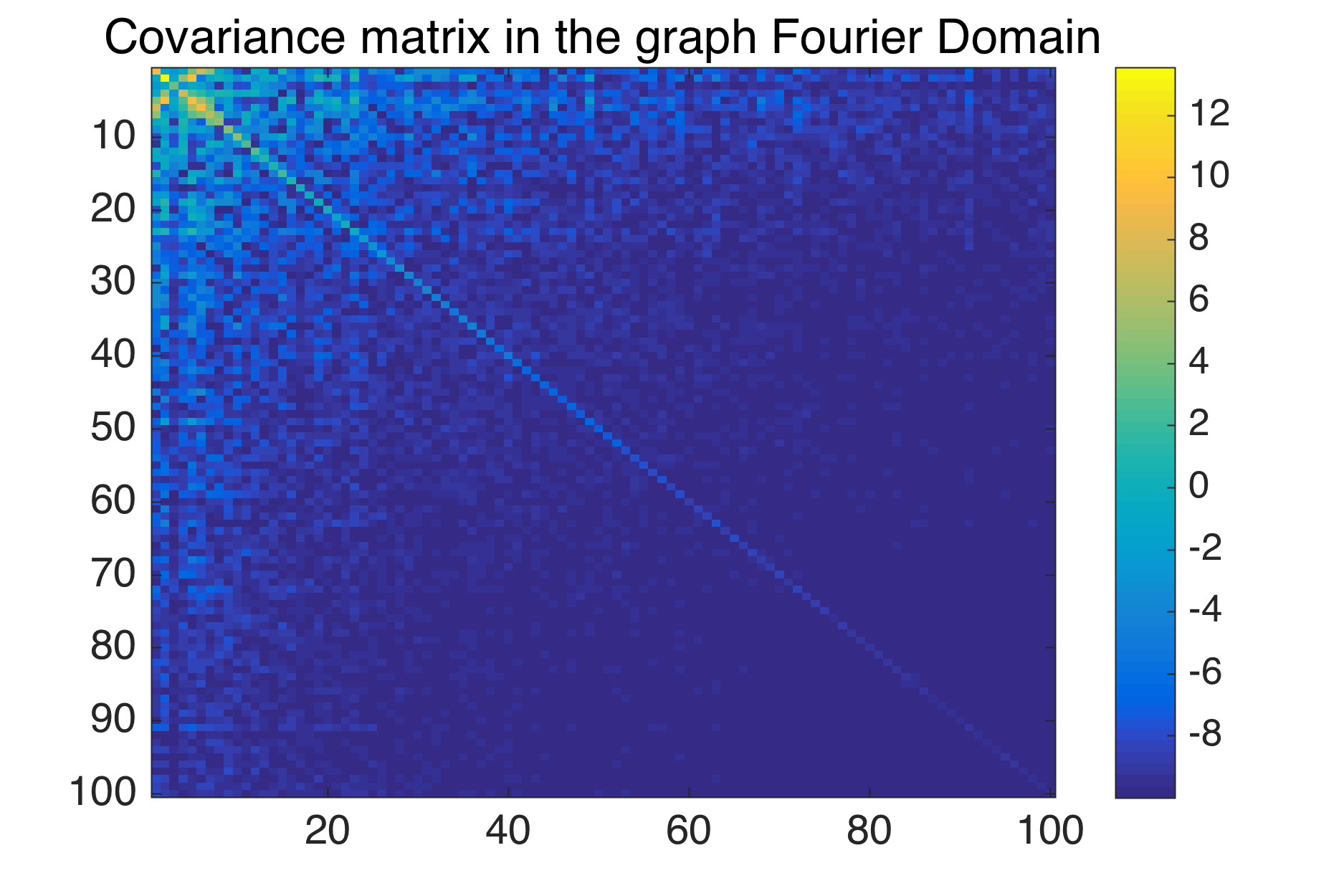}
\end{center}
\caption{Studying the full USPS dataset. Left: Covariance eigenvectors associated with the $16$ highest eigenvalues. Right: Laplacian eigenvectors associated to the $16$ smallest non-zero eigenvalues. Because of stationarity, Laplacian eigenvectors are similar to the covariance eigenvectors. Bottom: $\Gamma = P^\top C P$ in dB. Note the diagonal shape of the matrix implying that $P$ is aligned with the eigenvectors of $C$.}
\label{fig:exp4}
\end{figure}

From this example, we can conclude that every column of a  low-rank matrix $X$ lies approximately in the span of the eigenvectors $P_{k_2}$ of the features graph, where $k_2$ denotes the eigenvectors corresponding to the smallest $k_2$ eigenvalues. This is similar to PCA, where a low-rank matrix is represented in the span of the first few principal directions or atoms of the basis. Alternately,  the Laplacian eigenvectors are meaningful features for the USPS dataset. Let the eigenvectors $P$ of $\Larg_2$ be divided into two sets $(P_{k_2} \in \Re^{p\times k_2}, \bar{P}_{k_2} \in \Re^{p \times (p - k_2)})$. Note that the columns of $P_{k_2}$ contain the eigenvectors corresponding to the low graph frequencies and $\bar{P}_{k_2}$ contains those corresponding to higher graph frequencies. Then we can write, $X = X^{*} + E$, where $X^{*}$ is the low-rank part and $E$ models the noise or corruptions. Thus,
$$X = P_{k_2} A + \bar{P}_{k_2}\bar{A} ~ ~\text{and} $$  $$X^{*} =  P_{k_2} A  $$
where $A\in \Re^{k_2 \times n}$ and $\bar{A} \in \Re^{(p - k_2) \times n}$. From Fig.~\ref{fig:exp3} it is also clear that $\|\bar{P}_{k_2}\bar{A}\|_F \ll \|P_{k_2} A\|_F$ for a specific value of $k_2$.

\subsection{The graph of samples provides embedding for data}
The smallest eigenvectors of the graph of samples provide an embedding of the data in the low-dimensional space \cite{belkin2003laplacian}. This has a similar interpretation as the principal components in PCA. We argue that every row of a  low-rank matrix lies in  the span of the first few eigenvectors of the graph of samples. This is similar to representing every row of the low-rank matrix as the span of the principal components. Thus, the graph of samples $\Larg_1$ encodes a smooth non-linear map towards the principal components  of the underlying manifold defined by the graph $\Larg_1$. In other words, minimization with respect to  $\tr(U \Larg_1 U^\top )$ forces  the principal components of the data to be aligned with the eigenvectors $Q$ of the graph $\Larg_1$ which correspond to the smallest eigenvalues $\lambda_j$. This is the heart of  many algorithms in clustering \cite{ng2002spectral} and dimensionality reduction \cite{belkin2003laplacian}. In our present application this term has two effects. Firstly, when the data has a class structure, the graph of samples enforces the low-rank $U$ to benefit from this class structure. This results in an enhanced clustering of the low-rank signals. Secondly, it will force that the low-rank $U$  of the signals is well represented by the first few Laplacian eigenvectors associated to low $\lambda_{j}$. Let the eigenvectors $Q$ of $\Larg_1$ be divided into two sets $(Q_{k_1} \in \Re^{n\times k_1}, \bar{Q}_{k_1} \in \Re^{n \times (n - k_1)})$, where $k_1$ denotes the eigenvectors in $Q$ corresponding to the smallest $k_1$ eigenvalues. Note that the columns of $Q_{k_1}$ contain the eigenvectors corresponding to the low graph frequencies and $\bar{Q}_{k_1}$ contains those corresponding to higher graph frequencies. Then, we can write:
$$X = BQ^{\top}_{k_1} + \bar{B}\bar{Q}^{\top}_{k_1}  ~ ~\text{and} $$ $$X^{*} = BQ^{\top}_{k_1}$$
where $B\in \Re^{p \times k_1}$ and $\bar{B} \in \Re^{p \times (n - k_1)}$. As argued in the previous subsection, $\|\bar{B}\bar{Q}^{\top}_{k_1}\|_F   \ll  \| BQ^{\top}_{k_1} \|_F$.

\subsection{Low-rank matrix on graphs}\label{sec:LR}
From the above explanation related to the role of the two graphs, we can conclude the following facts about the representation of any clusterable low-rank  matrix $X^*$.
\begin{enumerate}
\item It can be represented as a linear combination of the  Laplacian eigenvectors of the graph of features, i.e, $X^* = P_{k_2} A $.
\item It can also be represented as a linear combination of the Laplacian eigenvectors of the graph of samples, i.e, $X^* = BQ^{\top}_{k_1} $. 
\end{enumerate}

As already pointed out, only the first $k_1$ or $k_2$ eigenvectors of the graphs correspond to the low frequency information, therefore, the other eigenvectors correspond to noise. We are now in a position to define \textit{low-rank matrix on graphs}.

\begin{defn}
A matrix $X^*$ is $(k_1, k_2)$-low-rank on the graphs $\Larg_1$ and $\Larg_2$ if $(X^*)_i^\top \in {\rm span}(Q_{k_1})$ for all $i = 1, \ldots, p$, and $(X^*)_j \in {\rm span}(P_{k_2})$ for all $j = 1, \ldots, n$. The set of  $(k_1, k_2)$-low-rank matrices on the graphs $\Larg_1$ and $\Larg_2$ is denoted by $\mathcal{LR}(Q_{k_1}, P_{k_2})$.
\end{defn}

We note here that $X^* \in {\rm span}(P_{k_2})$ means that the columns of $X^*$ are in ${\rm span}(P_{k_2})$, i.e, $(X^*)_i \in {\rm span}(P_{k_2})$, for all $i = 1, \ldots, n$, where for any matrix $A$, $(A)_i$ is its $i^{th}$ column vector.


\subsection{Theoretical Analysis}
The lower eigenvectors $Q_{k_1}$ and $P_{k_2}$ of $\Larg_1$ and $\Larg_2$ provide features for any $X \in \mathcal{LR}(Q_{k_1}, P_{k_2})$. Now we are ready to formalize our findings mathematically and prove that any solution of \eqref{eq:proposed} yields an approximately low-rank matrix. In fact, we prove this for  any proper, positive, convex and lower semi-continuous loss function $\phi$ (possibly $\ell_p$-norms $\|\cdot\|_1$, $\|\cdot\|_2^2$, ..., $\|\cdot\|_p^p$). We re-write \eqref{eq:proposed} with a general loss function $\phi$
\begin{equation}\label{eq:optim}
\min_{U} \phi(U - X) + \gamma_1 \tr(U \Larg_1 U^\top) + \gamma_2 \tr(U^\top \Larg_2 U)
\end{equation}

 Before presenting our mathematical analysis we gather a few facts which will be used later:
\begin{itemize}
\item We assume that the observed data matrix $X$ satisfies $X = X^* + E$ where $X^* \in \mathcal{LR}(Q_{k_1}, P_{k_2})$ and $E$ models noise/corruptions. Furthermore, for any $X^* \in \mathcal{LR}(Q_{k_1}, P_{k_2})$ there exists a matrix $C$ such that $X^* = P_{k_2} C Q_{k_1}^\top$.
\item $\Larg_1 = Q\Lambda Q^\top = Q_{k_1} \Lambda_{k_1} Q^{\top}_{k_1} + \bar{Q}_{k_1} \bar{\Lambda}_{k_1} \bar{Q}^{\top}_{k_1} $, where $\Lambda_{k_1} \in \Re^{k_1 \times k_1}$ is a diagonal matrix of lower eigenvalues  and    $\bar{\Lambda}_{k_1} \in \Re^{(n - k_1) \times (n - k_1)}$ is also a diagonal matrix of higher graph eigenvalues. All values in $\Lambda$ are sorted in increasing order, thus $0  = \lambda_0 \leq \lambda_1 \leq \cdots \leq \lambda_{k_1} \leq \cdots \leq \lambda_{n-1} $. The same holds for $\Larg_2$ as well.
\item For a $K$-nearest neighbors graph constructed from a $k_1$-clusterable data (along samples) one can expect $\lambda_{k_1}/\lambda_{k_1 + 1} \approx 0$ as $\lambda_{k_1} \approx 0$ and $\lambda_{k_1} \ll \lambda_{k_1 + 1}$. The same holds for the graph of features $\Larg_2$ as well.
\item For the proof of the theorem, we will use the fact that for any $X \in \Re^{p \times n}$, there exist $A \in \Re^{k_2 \times n}$ and $\bar{A} \in \Re^{(n-k_2) \times n}$ such that $X = P_{k_2} A + \bar{P}_{k_2} \bar{A}$, and $B \in \Re^{p \times k_1}$ and $\bar{B} \in \Re^{p \times (n-k_1)}$ such that $X = B Q_{k_1}^\top + \bar{B} \bar{Q}_{k_1}^\top$.
\end{itemize}

\begin{thm}\label{thm:gfrpcagnoisy} 
Let $X^* \in \mathcal{LR}(Q_{k_1}, P_{k_2})$, $\gamma>0$, and $E \in \Re^{p \times n}$. Any solution $U^*  \in \Re^{p \times n}$ of \eqref{eq:optim} with $\gamma_1 = \gamma/\lambda_{k_1+1}$, $\gamma_2 = \gamma/\omega_{k_2+1}$ and $X = X^* + E$ satisfies
\begin{align}\label{eq:bound}
& \phi(U^* - X) + \gamma_1 \| U^{*} \bar{Q}_{k_1}\|_F^2 + \gamma_2 \|\bar{P}_{k_2}^\top  U^*\|_F^2 \nonumber \\
&  \leq \phi(E) + \gamma \|X^*\|_F^2 \Big( \frac{\lambda_{k_1}}{\lambda_{k_1+1}} + \frac{\omega_{k_2}}{\omega_{k_2+1}} \Big).
\end{align}
where $\lambda_{k_1}, \lambda_{k_{1}+1} $ denote the $k_1, k_1 + 1 $ eigenvalues of ${\Larg}_1$,  $\omega_{k_2}, \omega_{k_{2}+1}$ denote the $k_2, k_2 + 1 $ eigenvalues of ${\Larg}_2$.
\end{thm}
\begin{proof}
As $U^*$ is a solution of \eqref{eq:optim}, we have
\begin{align}\label{eq:optim_bound}
& \phi(U^* - X) + \gamma_1 \tr(U^* \Larg_1 (U^*)^\top) + \gamma_2 \tr((U^*)^\top \Larg_2 U^*) \nonumber \\
 & \leq 
\phi(E) +  \gamma_1 \tr(X^* L_1 (X^*)^\top) + \gamma_2 \tr((X^*)^\top \Larg_2 X^*).
\end{align}
Using the facts that $\Larg_1 = Q_{k_1} \Lambda_{k_1} Q_{k_1}^\top + \bar{Q}_{k_1} \bar{\Lambda}_{k_1} \bar{Q}_{k_1}^\top$ and that there exists $B \in \Re^{p \times k_1}$ and $\bar{B} \in \Re^{p \times (n-k_1)}$ such that $U^* = B Q_{k_1}^\top + \bar{B} \bar{Q}_{k_1}^\top$, we obtain
\begin{align*}
& \tr(U^* \Larg_1 (U^*)^\top) 
= \tr(B  \Lambda_{k_1}  B^\top) + \tr(\bar{B}  \bar{\Lambda}_{k_1} \bar{B}^\top)
\nonumber \\
& \geq \tr(\bar{\Lambda}_{k_1} \bar{B}^\top \bar{B}) \geq \lambda_{k_1+1}\|\bar{B}\|_F^2  = \lambda_{k_{1}+1} \|U^* \bar{Q}_{k_1}\|_F^2.
\end{align*}
Then, using the fact that there exists $C \in \Re^{k_2 \times k_1}$ such that $X^* = P_{k_2} C Q_{k_1}^\top$, we obtain
\begin{align*}
\tr(X^* L_1 (X^*)^\top) = \tr(C \Lambda_{k_1} C^\top) \leq \lambda_{k_1} \|C\|_F^2 = \lambda_{k_1} \|X^*\|_F^2.
\end{align*}
Similarly, we have
\begin{align*}
\tr((U^*)^\top \Larg_2 U^*) 
\geq
\omega_{k_2+1} \|\bar{P}_{k_2}^\top U^*\|_F^2,
\end{align*}
\begin{align*}
\tr((X^*)^\top \Larg_2 X^*) \leq \omega_{k_2} \|X^*\|_F^2.
\end{align*}
Using the four last bounds in \eqref{eq:optim_bound} yields
\begin{align*}
& \phi(U^* - X) + \gamma_1 \lambda_{k_1+1} \|U^* \bar{Q}_{k_1}\|_F^2 + \gamma_2 \omega_{k_2+1} \|\bar{P}_{k_2}^\top U^*\|_F^2 
\leq \nonumber \\
& \phi(E) +  \gamma_1 \omega_{k_1} \|X^*\|_F^2 + \gamma_2 \omega_{k_2} \|X^*\|_F^2,
\end{align*}
which becomes 
\begin{align*}
& \phi(U^* - X) + \gamma \|U^* \bar{Q}_{k_1}\|_F^2 + \gamma \|\bar{P}_{k_2}^\top U^*\|_F^2 \nonumber \\
& \leq 
\phi(E) +  \gamma \|X^*\|_F^2 \left( \frac{\lambda_{k_1}}{\lambda_{k_1+1}}  + \frac{\omega_{k_2}}{\omega_{k_2+1}} \right)
\end{align*}
for our choice of $\gamma_1$ and $\gamma_2$. This terminates the proof.
\end{proof}

\subsection{Remarks on the theoretical analysis}
\eqref{eq:bound} implies that
\begin{align*}
& \|U^* \bar{Q}_{k_1}\|_F^2 + \|\bar{P}_{k_2}^\top U^*\|_F^2  \leq 
\frac{1}{\gamma} \phi(E) +  \|X^*\|_F^2 \left( \frac{\lambda_{k_1}}{\lambda_{k_1+1}}  + \frac{\omega_{k_2}}{\omega_{k_2+1}} \right).
\end{align*}

The smaller $\|U^* \bar{Q}_{k_1}\|_F^2 + \|\bar{P}_{k_2}^\top U^*\|_F^2$ is, the closer $U^*$ to $\mathcal{LR}(Q_{k_1}, P_{k_2})$ is. The above bound shows that to recover a low-rank matrix one should have large eigengaps ${\lambda_{k_1+1}}-{\lambda_{k_1}}$ and ${\omega_{k_2+1}} - {\omega_{k_2}}$. This occurs when the rows and columns of $X$ can be clustered into $k_1$ and $k_2$ clusters. Furthermore, one should also try to chose a metric $\phi$ (or $\ell_p$-norm) that minimizes $\phi(E)$. Clearly, the rank of $U^{*}$ is approximately $\min\{k_1,k_2\}$.

\section{Working of FRPCAG}\label{sec:working}

The previous section presented a theoretical analysis of our model. In this section, we explain in detail the working of our model for any data. Like any standard low-rank recovery method, such as \cite{candes2011robust}, our method is able to perform the following two operations for a clusterable data:
\begin{enumerate}
\item Penalization  of the singular values of the data. The penalization of the  higher singular values, which correspond to high frequency components in the data results in the data cleaning. 
\item Determination of the clean left and right singular vectors. 
\end{enumerate}

\subsection{FRPCAG is a singular value penalization method}
In order to demonstrate how FRPCAG penalizes the  singular values of the data we study another  way to cater the graph regularization in the solution of the optimization problem which is contrary to the one presented in Section~\ref{sec:optimization}. In Section \ref{sec:optimization} we used a gradient for the graph regularization terms $\gamma_{1}\tr(U\Larg_{1}U^\top) + \gamma_2 \tr(U^\top \Larg_{1}U)$ and used this gradient as an argument of the proximal operator for the soft-thresholding.  What we did not point out there was that the solution of the graph regularizations can also be computed by proximal operators. It is due to the reason that using proximal operators for graph regularization (that we present here) is more computationally expensive.   Assume that the prox of $\gamma_{1}\tr(U\Larg_{1}U^\top)$ is computed first, and let $Z$ be a temporary variable, then it can be written as: 
\begin{equation*}
\min_{Z} \|X-Z\|^2_F + \gamma_{1}\tr(Z\Larg_{1}Z^\top)
\end{equation*}
The above equation has a closed form solution which is given as:
$$Z = X(I + \gamma_1 \Larg_1)^{-1}$$
Now, compute the proximal operator for the term $\gamma_{2}\tr(U^\top \Larg_{2}U)$
\begin{equation*}
\min_{U} \|Z-U\|^2_F + \gamma_{2}\tr(U^\top \Larg_{2}U)
\end{equation*}
The closed form solution of the above equation is given as:
$$U = (I + \gamma_2 \Larg_2)^{-1} Z$$
Thus, the low-rank $U$ can be written as:
$$U = (I + \gamma_2 \Larg_2)^{-1} X (I + \gamma_1 \Larg_1)^{-1} $$
after this the soft thresholding can be applied on $U$.

Let the SVD of $X$, $X = V_x\Sigma_x W^{\top}_x $, $\Larg_1 = Q\Lambda Q^\top$ and $\Larg_2 = P\Omega P^\top$, then we get:
\begin{align*} 
 U  & = (I + \gamma_2 P\Lambda P^\top)^{-1} V_x\Sigma_x W^{\top}_x (I + \gamma_1 Q\Omega Q^\top)^{-1} \nonumber \\
 & = P(I + \gamma_2\Lambda)^{-1} P^\top V_x\Sigma_x W^{\top}_x Q(I + \gamma_1 \Omega)^{-1} Q^\top
\end{align*}
thus, each singular value $\sigma_{xi}$ of $X$ is penalized by $1/(1+\gamma_1 \lambda_i)(1+\gamma_2 \omega_i)$. Clearly, the above solution requires the computation of two inverses which can be computationally intractable for big datasets. 
 
 \subsection{Estimation of clean singular vectors}
The two graph regularization terms $\tr(U\Larg_{1} U^\top )$  $\tr(U^{\top}\Larg_{2} U^\top )$  encode a weighted penalization in the Laplacian basis. Again using $\Larg_{1} = Q\Lambda Q^\top$ and $\Larg_{2} = P\Omega P^\top$ and $U=V\Sigma W^{\top}$ be the SVD of $U$, we get
\begin{align}\label{eq:original}
\gamma_{1}&\tr(U\Larg_{1}U^\top) + \gamma_{2}\tr(U^\top \Larg_{2}U)\nonumber\\
= &  \gamma_{1}\tr(V\Sigma W^{\top}Q\Lambda Q^{\top}W\Sigma V^{\top}) + \gamma_{2}\tr(W\Sigma V^{\top}P\Omega P^{\top}V\Sigma W^{\top}) \nonumber \\
= & \gamma_{1}\tr(\Sigma W^{\top}Q\Lambda Q^{\top}W\Sigma) + \gamma_{2}\tr(\Sigma V^{\top}P\Omega P^{\top}V\Sigma) \nonumber \\
= & \gamma_{1}\tr(W^{\top}Q\Lambda Q^{T}W\Sigma^{2}) + \gamma_{2}\tr(V^{\top}P\Omega P^{\top}V\Sigma^{2}) \nonumber \\
= & \sum_{i,j = 1}^{\min\{n,p\}}\sigma^{2}_{i}(\gamma_{1}\lambda_j (w^{\top}_{i}q_j)^{2}+\gamma_{2}\omega_{j}(v^{\top}_{i}p_{j})^{2}),
\end{align} 

where $\lambda_j$ and $\omega_j$ are the eigenvalues in the matrices $\Lambda$ and $\Omega$ respectively. The second step follows from $V^{\top}V = I$ and the cyclic permutation invariance of the trace. In the standard terminology $w_i$ and $v_{i}$ are the principal components and principal directions of of the low-rank matrix $U$. From the above expression, the minimization is carried out with respect to the singular values $\sigma_i$ and the singular vectors $v_i, w_i$. The minimization has the following effect:
\begin{enumerate}
\item Minimize $\sigma_i$ by performing an attenuation with the graph eigenvalues as explained earlier.
\item When $\sigma_i$ is big, the principal components $w_i$ are aligned with the graph eigenvectors $q_j$ for small values of $\lambda_{j}$, i.e, the lower graph frequencies of $\Larg_1$. The principal directions $v_i$ are also aligned with the graph eigenvectors $p_j$ for small values of $\omega_{j}$, i.e, the lower graph frequencies of $\Larg_2$. This alignment makes sense as the higher eigenvalues correspond to the higher graph frequencies which constitute the noise in data. This is explained experimentally in the next section.
\end{enumerate}

\subsection{Experimental Justification of working of FRPCAG}\label{sec:aprox_lowrank}
Now we  present an experimental justification for the working  of this model, as described in the previous subsection. In summary we illustrate that:
\begin{enumerate}
\item The model recovers a close-to-low-rank representation.
\item The principal components and principal directions of $U$ align with the first few eigenvectors of their respective graphs, automatically revealing a low-rank and enhanced class structure.
\item  The singular values of the low-rank matrix obtained using our model closely approximate those obtained by nuclear norm based models even  in the presence of corruptions.
\end{enumerate}

Our justification relies mostly on the quality of the singular values of the low-rank representation and the alignment of the singular vectors with their respective graphs.  

%
    
     \begin{figure*}[htbp]
    \centering
        \centering
        \includegraphics[width=1.0\textwidth]{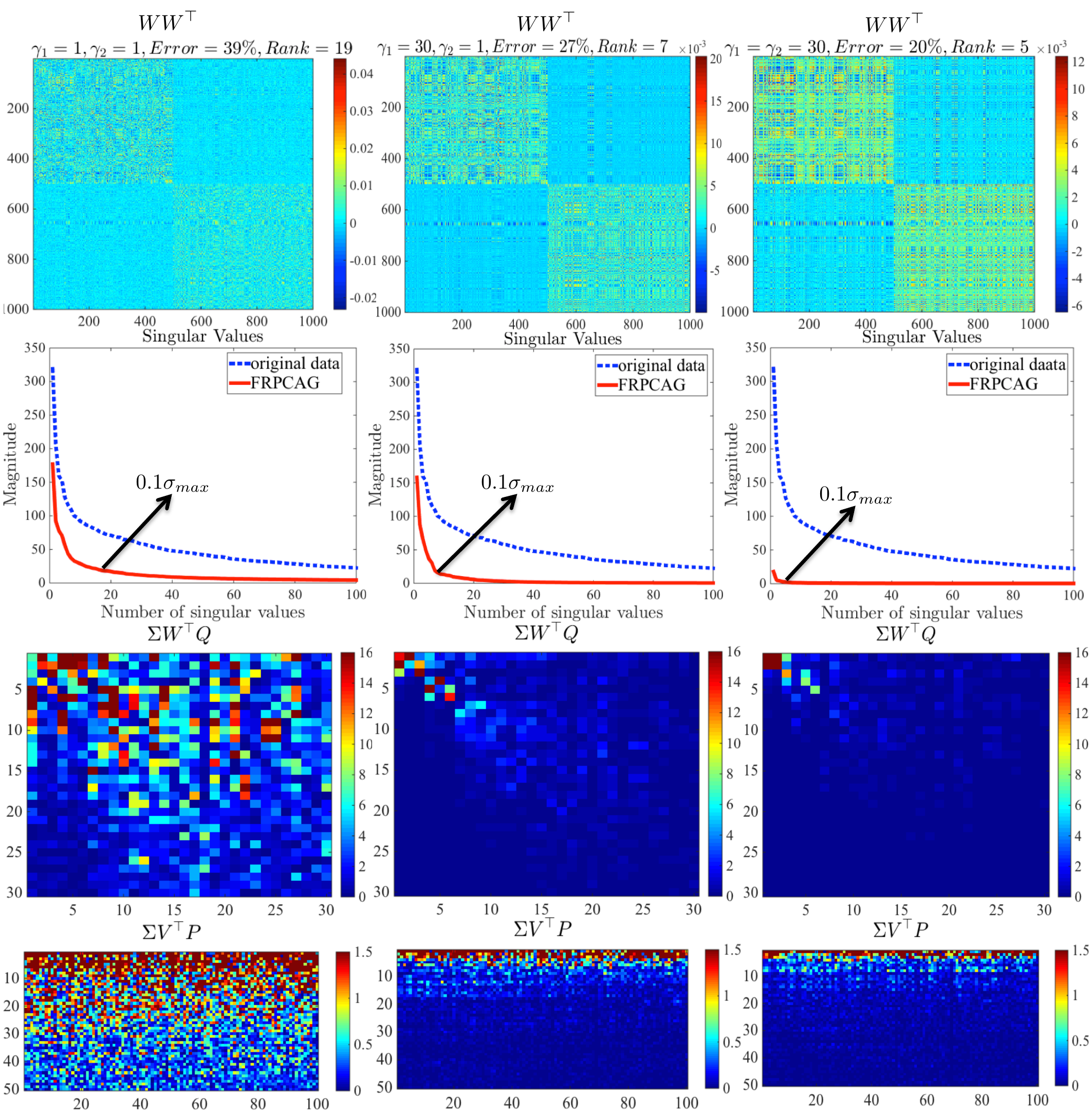}
         \caption{The matrices $WW^\top$, $\Sigma$, ${\Sigma}W^{\top}Q$, ${\Sigma}V^{\top}P$ and the corresponding clustering errors obtained for different values of the weights on the two graph regularization terms for 1000 samples of MNIST dataset (digits 0 and 1). If $U=V\Sigma W^{\top}$ is the SVD of $U$, then $W$ corresponds to the matrix of principal components (right singular vectors of $U$) and $V$ to the principal directions (left singular vectors of $U$). Let $\Larg_1 = Q\Lambda Q^{\top}$ and $\Larg_2 = P\Omega P^{\top}$ be  the eigenvalue decompositions of $\Larg_1$ and $\Larg_2$ respectively then $Q$ and $P$ correspond to the eigenvectors of Laplacians $\Larg_1$ and $\Larg_2$.  The block diagonal structure of $WW^{\top}$ becomes more clear by increasing $\gamma_{1}$ and $\gamma_{2}$ with a thresholding of the singular values in $\Sigma$. Further, the sparse structures of ${\Sigma}W^{\top}Q$ and ${\Sigma}V^{\top}P$ towards the rightmost corners show that the  number of left and right singular vectors which align with the eigenvectors of the Laplacians $\Larg_1$ and $\Larg_2$ go on decreasing with increasing $\gamma_1$ and $\gamma_2$. This shows that the two graphs help in attaining a low-rank structure with a low clustering error.}
        \label{fig:subspaces}
    \end{figure*}

We perform an experiment with 1000 samples of the MNIST dataset belonging to two different classes (digits 0 and 1). We vectorize all the digits and form a data matrix $X$ whose columns contain the digits. Then we compute a graph of samples between the columns of $X$ and a graph of features between the rows of $X$ as mentioned in Section \ref{sec:graphs}. We determine the clean low-rank $U$ by solving model \eqref{eq:proposed} and perform one SVD at the end $U = V\Sigma W^\top$. Finally, we do the  clustering  by performing k-means (k = 2) on the low-rank $U$. As argued in \cite{liu2013robust}, if the data is arranged according to the classes, the matrix $WW^\top $ (where $W$ are the principal components of the data) reveals the subspace structure. The matrix $WW^\top$ is also known as the shape interaction matrix (SIM) \cite{ji2015shape}.  If the subspaces are orthogonal then SIM  should acquire a block diagonal structure.  Furthermore, as explained in Section~\ref{sec:proposed_model} our model \eqref{eq:proposed} tends to align the first few principal components $w_i$ and principal directions $v_i$ of $U$ to the first few eigenvectors  $q_j$  and $p_j$ of $\Larg_1$ and $\Larg_2$ respectively. Thus, it is interesting to observe the matrices ${\Sigma}W^{\top}Q$ and ${\Sigma}V^{\top}P$ scaled with the singular values $\Sigma$ of the low-rank matrix $U$, as justified by eq. (\ref{eq:original}). This scaling takes into account the importance of the eigenvectors that are associated to bigger singular values.

Fig.~\ref{fig:subspaces} plots the matrix $WW^\top $, the corresponding clustering error, the matrices $\Sigma$, ${\Sigma}W^{\top}Q$, and ${\Sigma}V^{\top}P$  for different values of $\gamma_{1}$ and $\gamma_{2}$ from left to right. Increasing $\gamma_1$ and $\gamma_2$ from 1 to 30 leads to 1) the penalization of the singular values in $\Sigma$ resulting in a lower rank 2) alignment of the first few principal components $w_i$ and principal directions $v_i$ in the direction of the first few eigenvectors $q_j$ and $p_j$ of $\Larg_1$ and $\Larg_2$ respectively 3) an enhanced subspace structure in $WW^{\top}$ and 4) a lower clustering error. Together the two graphs help in acquiring a low-rank structure that is suitable for clustering applications as well. 

   \begin{figure*}[htbp]
    \centering
        \centering
        \includegraphics[width=1.0\textwidth]{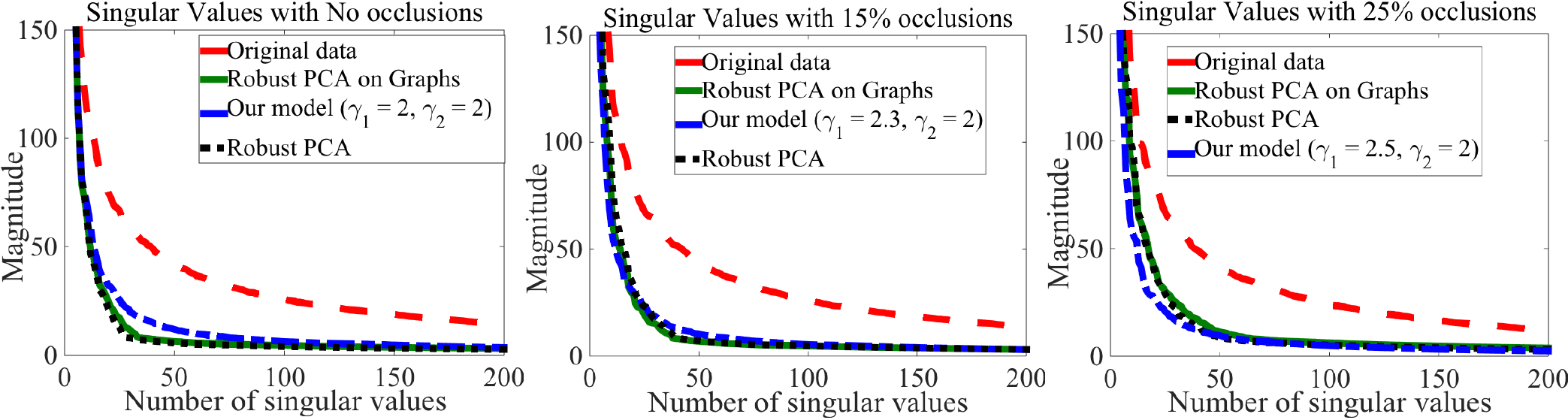}
     \caption{A comparison of singular values of the low-rank matrix obtained via our model, RPCA and RPCAG. The experiments were performed on the ORL dataset with different levels of block occlusions. The parameters corresponding to the minimum validation clustering error for each of the model were used.}
        \label{fig:svs_occlusions}
    \end{figure*}

    Next we demonstrate that for data with or without corruptions, FRPCAG is able to acquire singular values as good as the nuclear norm based models, RPCA and RPCAG. We perform three clustering experiments on 30 classes of ORL dataset with no block occlusions, 15\% block occlusions and 25\% block occlusions.  Fig.~\ref{fig:svs_occlusions} presents a comparison of the singular values of the original data with the singular values of the low-rank matrix obtained by solving RPCA, RPCAG and our model. The parameters for all the models are selected corresponding to the lowest clustering error for each model.  It is straightforward to conclude that the singular values of the low-rank representation using our fast method closely approximate those of the nuclear norm based models irrespective of the level of corruptions.
 

\section{Computational Complexity}\label{sec:complexity_o}
\subsection{Complexity of Graph Construction} For $n$ $p$-dimensional vectors, the computational complexity of the FLANN algorithm is $\mathcal{O}(pn K (\log(n)/\log(K)))$ for the graph $G_1$ between the samples and $\mathcal{O}(pn K (\log(p)/\log(K)))$ for the graph $G_2$ between the features, where $K$ is the number of nearest neighbors. This has been shown in  \cite{sankaranarayanan2007fast} \& \cite{muja2014scalable} .  For a fixed $K$ the complexity of $G_1$ is $\mathcal{O}(pn\log(n))$ and that of $G_2$ is $\mathcal{O}(np\log(p))$. We use $K=10$ for all the experiments reported in this work. The effect of the variation of $K$ on the performance of our model is studied briefly in Section~\ref{sec:experiments}.


\subsection{Algorithm Complexity}
\subsubsection{FISTA}
 Let $I$ denote the number of iterations for the algorithm to converge, $p$ is the data dimension, $n$ is the number of samples and $c$ is the rank of the low-dimensional space. The computational cost of our algorithm per iteration is linear in the number of data samples $n$, i.e.  $\mathcal{O}(Ipn)$ for $I$ iterations. 
 
\subsubsection{Final SVD}
Our model, in order to preserve convexity, finds an approximately low-rank solution $U$ without explicitly factorizing it. While this gives a great advantage, depending on the application we have in hand, we might need to provide explicitly the low dimensional representation in a factorized form. This can be done by computing an ``economic'' SVD of $U$ after our algorithm has finished. 

Most importantly, this computation can be done in time that scales linearly with the number of samples for a fixed number of features $p\ll n$. Let 
$U=V\Sigma W^{\top}$ the SVD of $U$. The orthonormal basis $V$ can be computed by the eigenvalue decomposition of the small $p\times p$ matrix $UU^\top = VEV^\top$ that also reveals the singular values $\Sigma = \sqrt{E}$ since $UU^\top$ is s.p.s.d. and therefore $E$ is non-negative diagonal. Here we choose to keep only the  $c$ biggest singular values and corresponding vectors according to the application in hand (the procedure for determining $c$ is explained in Section~\ref{sec:experiments}). Given $V$ and $\Sigma$ the sample projections are computed as $W = \Sigma^{-1}V^\top U$.

The complexity of this SVD is $\mathcal{O}(np^2)$ --due to the multiplication $UU^\top$-- and does not change the asymptotic complexity of our algorithm. Note that the standard economic SVD implementation in numerical analysis software typically does not use this simple trick, in order to achieve better numerical error. However, in most machine learning applications like the ones of interest in this paper, the compromise in terms of numerical error is negligible compared to the gains in terms of scalability.

\subsection{Overall Complexity}
The complexity of FISTA is $\mathcal{O}(Ipn)$, the graph $G_1$ is $\mathcal{O}(pn\log(n))$, $G_2$ is $\mathcal{O}(pn\log(p))$ and the final SVD step is $\mathcal{O}(np^2)$. Given that $p\ll n$, the overall complexity of our algorithm is $\mathcal{O}(pn(\log(n)+I+p + \log(p)))$. Table~\ref{tab:complexity}  presents the computational complexities of all the models considered in this work (discussed in Section~\ref{sec:experiments}).

\begin{table*}[htbp]
\caption{Computational complexity of all the models considered in this work. $I$ denotes the number of iterations for the algorithm to converge, $p$ is the dimension, $n$ is the number of samples and $c$ is the rank of the low-dimensional space.  All the models which use the graph $G_1$ are marked by '+'. The construction of graph $G_2$ is included only in our model (FRPCAG). Note that we have used complexity $\mathcal{O}(np^2)$ for all SVD computations and $\mathcal{O}(In)$ for approximate eigenvalue decomposition for NCut as proposed in \cite{shi2000normalized}, while the latter could be used for the decomposition needed by LE even though not specified in \cite{belkin2003laplacian}.}
\centering
\resizebox{0.9\textwidth}{!}{\begin{tabular}[t]{| c || c | c | c | c |}
\hline
\textbf{Model}  & \textbf{Complexity $G_1$} & \textbf{Complexity $G_2$} & \textbf{Complexity Algorithm} & \textbf{Overall Complexity} \\
                 &  $\mathcal{O}(np \log(n))$  & $\mathcal{O}(np\log(p))$  &for $p\ll n$  & for  $p\ll n$\\\hline\hline
                  FRPCAG  & + & + & $\mathcal{O}(np(I+p)) $ & $\mathcal{O}(np(\log(n)+p+I+\log(p)))$\\\hline
 NCut \cite{shi2000normalized}  & + &  -- & $\mathcal{O}(In)$ & $\mathcal{O}(n(p\log(n)+I))$ \\\hline 
LE  \cite{belkin2003laplacian}   &  + & -- & $\mathcal{O}(n^{3})$  &  $\mathcal{O}(n(p\log(n)+n^2))$ \\\hline 
PCA   &  -- & -- & $\mathcal{O}(p^{2}n)$ & $\mathcal{O}(np(p\log(n)+p))$ \\\hline
 GLPCA \cite{jiang2013graph} &  + & -- &  $\mathcal{O}(n^{3})$ &  $\mathcal{O}(n(p\log(n)+n^2))$ \\\hline
 NMF  \cite{lee1999learning}  & --  & -- & $\mathcal{O}(Inpc)$ &   $\mathcal{O}(Inpc)$ \\\hline
 GNMF  \cite{cai2011graph} & + & -- & $\mathcal{O}(Inpc)$  & $\mathcal{O}(np(Ic+\log(n)))$ \\\hline
 MMF  \cite{zhang2013low} & + & -- & $\mathcal{O}(((p+c)c^{2}+pc)I)$ & $\mathcal{O}(((p+c)c^{2}+pc)I+pn\log(n))$  \\\hline
 RPCA \cite{candes2011robust} & -- & --  & $\mathcal{O}(Inp^{2})$ & $\mathcal{O}(np(Ip+\log(n)))$\\\hline
 RPCAG \cite{shahid2015robust} & + & -- & $\mathcal{O}(Inp^{2})$ & $\mathcal{O}(np(Ip+\log(n)))$ \\\hline
\end{tabular}}
\label{tab:complexity}
\end{table*}

\subsection{Scalability}
The construction of graphs $G_1$ and $G_2$ is highly scalable. For small $n$ and $p$ the strategy 1 of Section~\ref{sec:graphs} can be used for the graphs construction and each of the entries of the adjacency matrix $A$ can be computed in parallel once the nearest neighbors have been found. For large $n$ an $p$ the approximate K-nearest neighbors scheme (FLANN) is used for graphs construction which is highly scalable as well. Next, our proposed FISTA algorithm for FRPCAG requires two important computations at every iteration: 1) computation of proximal operator $\prox_{\lambda h}(U)$ and 2) the gradient $\nabla g(Y)$. The former computation is given by the element-wise soft-thresholding (eq.~\eqref{eq:prox}) that can be performed in parallel for all the entries of a matrix.  The gradient computation, as given by eq.~\eqref{eq:grad}, involves matrix-matrix multiplications that involve sparse matrices $\Larg_1$ and $\Larg_2$ and can be performed very efficiently in parallel as well.

\section{Results}\label{sec:experiments}
Experiments were done using two open-source toolboxes: the UNLocBoX \cite{perraudin2014unlocbox} for the optimization part and the GSPBox \cite{perraudin2014gspbox} for the graph creation. The complete demo, code and datasets used for this work are available at \href{https://lts2.epfl.ch/research/reproducible-research/frpcag/}. We perform two types of experiments corresponding to two applications of PCA.
\begin{enumerate}
\item Data clustering in the low-dimensional space.
\item Low-rank recovery: Static background separation from videos.
\end{enumerate}
We present extensive quantitative results for clustering but currently our experiments for low-rank recovery are limited to qualitative analysis only. This is because our work on approximating the low-rank representation using graphs is the first of its kind.  The experiments on the low-rank background extraction from videos suffice as a proof-of-concept for the working of this model.

We perform our clustering experiments on 7 benchmark databases: 
CMU PIE, ORL, YALE, COIL20, MNIST, USPS and MFEAT.  CMU PIE, ORL and YALE are face databases with small pose variations. COIL20 is a dataset of objects with significant pose changes so we select the images for each object with less than 45 degrees of pose change. USPS and MNIST contain images of handwritten digits and MFeat consists of features extracted from handwritten numerals. The details of all datasets used are provided in Table~\ref{tab:datasets}. 


\begin{table}[htbp]
\caption{Details of the datasets used for clustering experiments in this work.}
\centering
\resizebox{0.4\textwidth}{!}{\begin{tabular}[t]{| c | c | c | c |}
\hline
 \textbf{Dataset}  & \textbf{Samples}  & \textbf{Dimension}  & \textbf{Classes} \\\hline
 CMU PIE   &  1200  & $32 \times 32$  & 30  \\\hline
 ORL    &  400   & $56 \times 46$   & 40  \\\hline
 COIL20  & 1400   &  $32 \times 32$  &  20  \\\hline
  YALE   & 165  &   $32 \times 32$  & 11  \\\hline
 MNIST    & 50000  & $28 \times 28$  & 10  \\\hline
 USPS   &  3500  &  $16 \times 16$  & 10  \\\hline
 MFEAT  & 400   & 409   & 10 \\\hline
\end{tabular}}
\label{tab:datasets}
\end{table}

In order to evaluate the robustness of our model to gross corruptions we corrupt the datasets with two different types of errors 1) block occlusions and 2) random missing pixels. Block occlusions of three different sizes, i.e, 15\%, 25\% and 40\% of the total size of the image are placed uniformly randomly in all the images of the datasets. Similarly, all the images of the datasets are also corrupted by removing 10\%, 20\%, 30\% and 40\% pixels uniformly randomly. Separate clustering experiments are performed for each of the different types of corruptions.

We compare the clustering performance of our model with 10 other models including the state-of-art: 1) k-means on original data 2) Normalized Cut (NCut) \cite{shi2000normalized} 3) Laplacian Eigenmaps (LE) \cite{belkin2003laplacian}  4) Standard PCA 5) Graph Laplacian PCA (GLPCA) \cite{jiang2013graph} 6) Manifold Regularized Matrix Factorization (MMF) \cite{zhang2013low}  7) Non-negative Matrix Factorization (NMF) \cite{lee1999learning} 8) Graph Regularized Non-negative Matrix Factorization (GNMF) \cite{cai2011graph} 9) Robust PCA (RPCA) \cite{candes2011robust} and 10) Robust PCA on Graphs (RPCAG) \cite{shahid2015robust}. For ORL, CMU PIE, COIL20, YALE and USPS datasets we compare three different versions of our model corresponding to the three types of graphs $G_1$ and $G_2$.

As mentioned in Section~\ref{sec:graphs}, FRPCAG(A) corresponds to our model using good quality sample and feature graphs, FRPCAG(B) to the case using a good quality sample graph and approximate feature graph, and FRPCAG(C) to the case where approximate graphs were used both between samples and features. All other models for these datasets are evaluated using a good quality sample graph $G_1$. Due to the large size of the MNIST dataset, we use FLANN (strategy 2) to construct both graphs, therefore we get approximate versions in the presence of corruptions. Thus the experiments on MNIST dataset are kept separate from the rest of the datasets to emphasize the difference in the graph construction strategy.  We also perform a separate set of experiments on the ORL dataset  and compare the performance of our model with state-of-the-art nuclear norm based models, RPCA \cite{candes2011robust} and RPCAG \cite{shahid2015robust}, both with a good quality and an approximate graph. We perform this set of experiments only on the ORL dataset (due to its small size) as the nuclear norm based models are computationally expensive. Finally, the experiments on MFeat dataset are only performed with missing values because block occlusions in non-image datasets correspond to an unrealistic assumption. The computational complexities of all these models are presented in Table~\ref{tab:complexity}.

\textbf{Pre-processing:} All datasets are transformed to zero-mean and unit standard deviation along the features for the RPCA, RPCAG and FRPCAG. For MMF the samples are additionally normalized to unit-norm. For NMF and GNMF only the unit-norm normalization is applied to all the samples of the dataset.

\textbf{Evaluation:} We use \textit{clustering error} as a metric to compare the clustering performance of various models. NCut, LE, PCA, GLPCA, MMF, NMF and GNMF are matrix factorization models that explicitly learn the principal components $W$. The clustering error for these models is evaluated by performing k-means on the principal components. RPCA, RPCAG and FRPCAG learn the low-rank matrix $U$. The clustering error for these models can be  evaluated by performing k-means on 1) principal components $W$ obtained by the SVD of the low-rank matrix $U = V\Sigma W^{\top}$ or 2) the low-rank $U$ directly. Note that RPCA and RPCAG determine the exact low-rank representation $U$, whereas our model only shrinks singular values and therefore only recovers an approximate low-rank representation $U$. Thus,  if one desires to use the principal components $W$ for clustering, the  dimension of the subspace (number of columns of W)  can be decided by selecting the number of singular values greater than a particular threshold. However, this procedure requires SVD and can be expensive for big datasets. Instead, it is more feasible to perform clustering on the low-rank $U$ directly.  We observed that similar clustering results are obtained by using either $W$ or $U$, however, for brevity these results are not reported.  Due to the non-deterministic nature of k-means, it is run $10$ times and the minimum error over all runs is reported.

\textbf{Parameter selection for various models:} Each model has several parameters which have to be selected in the validation stage of the experiment. To perform a fair validation for each of the models we use a range of parameter values as presented in Table~\ref{tab:models_param}. For a given dataset, each of the models is run for each of the parameter tuples in this table and the parameters corresponding to minimum clustering error are selected for testing purpose. Furthermore, PCA, GLPCA, MMF, NMF and GNMF are non-convex models so they are run $10$ times for each of the parameter tuple. RPCA, RPCAG and FRPCAG are convex so they are run only once.

\begin{table}[htbp]
\footnotesize
\caption{Range of parameter values for each of the models considered in this work. $c$ is the rank or dimension of subspace, $\lambda$ is the weight associated with the sparse term for Robust PCA framework \cite{candes2011robust} and $\gamma$ is the parameter associated with the graph regularization term.}
\centering
\resizebox{0.45\textwidth}{!}{\begin{tabular}[t]{| c | c | c | } \hline
  \textbf{Model}   & \textbf{Para-}   & \textbf{Parameter} \\
                   & \textbf{meters}  & \textbf{ Range}   \\\hline
    NCut  \cite{shi2000normalized} &  &   \\\cline{1-1}
 LE \cite{belkin2003laplacian} & $c$ &  $c\in \{2^{1},2^{2},\cdots, \min(n,p) \}$  \\\cline{1-1}
  PCA    & &    \\\hline
  GLPCA  \cite{jiang2013graph}  &  & $c\in \{2^{1},2^{2},\cdots, \min(n,p) \}$    \\
   & $c,\gamma$ & $\gamma \implies \beta$ using  \cite{jiang2013graph} \\
        &       &  $\beta \in \{0.1, 0.2, \cdots, 0.9\}$  \\\hline
MMF  \cite{zhang2013low}      & $c,\gamma$ &  $c\in \{2^{1},2^{2},\cdots, \min(n,p) \}$ \\\cline{1-2}
 NMF  \cite{lee1999learning} &  $c$ &   \\\cline{1-2}
 GNMF \cite{cai2011graph}     & $c,\gamma$ &  $\gamma \in \{2^{-3},2^{-2},\cdots, 2^{10}\}$    \\\hline
 RPCA \cite{candes2011robust}  & $\lambda$  & $\lambda \in \{\frac{2^{-3}}{\sqrt{\max(n,p)}}:$ \\
                               &          & $0.1:\frac{2^{3}}{\sqrt{\max(n,p)}}\}$   \\\cline{1-2}
 RPCAG \cite{shahid2015robust} & $\lambda, \gamma$ & $\gamma \in \{2^{-3},2^{-2},\cdots, 2^{10}\}$    \\\hline
 FRPCAG &  $\gamma_{1}, \gamma_{2}$ & $\gamma_{1},\gamma_{2} \in \{1,2,\cdots, 100\}$    \\\hline
\end{tabular}}
\label{tab:models_param}
\end{table}

\textbf{Parameter selection for Graphs:} For all the experiments reported in this paper we use the following parameters for graphs $G_1$ and $G_2$. K-nearest neighbors = 10 and  $\sigma^{2} =1$. It is important to point out here that different types of data might call for slightly different parameters for graphs. However, for a given dataset, the use of same graph parameters (same graph quality) for all the graph regularized models ensures a fair comparison. 

\subsection{Clustering}

\begin{figure*}[htbp]
    \centering
        \centering
        \includegraphics[width=1.0\textwidth]{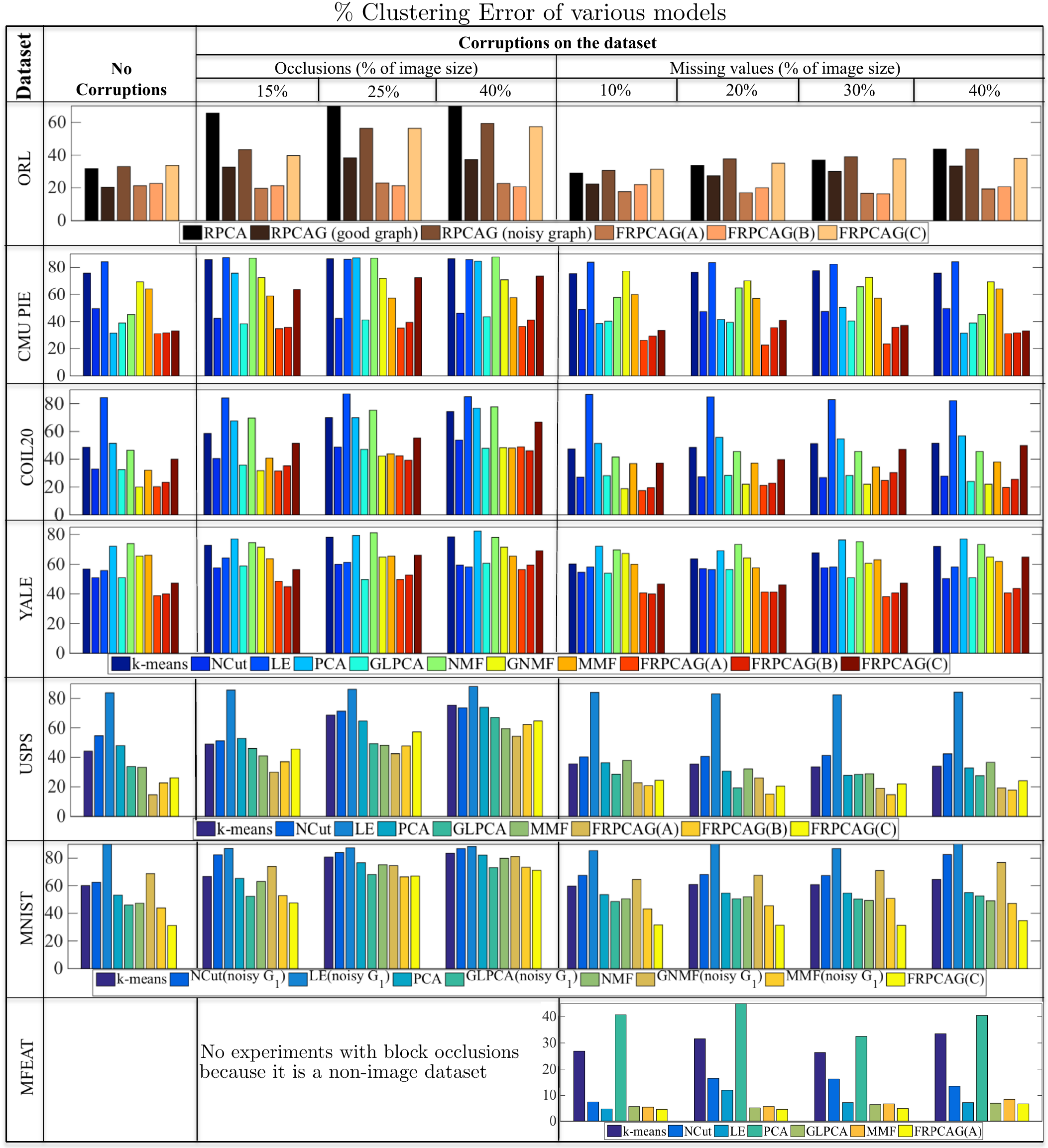}
        \caption{A comparison of  clustering  error of our model with various dimensionality reduction models. The image data sets include: 1) ORL 2) CMU PIE  3) COIL20 and 4) YALE. The compared models are: 1) k-means 2) Normalized Cut (NCut) 3) Laplacian Eigenmaps (LE) \cite{belkin2003laplacian}  4) Standard Principal Component Analysis (PCA)  5) Graph Laplacian PCA (GLPCA) \cite{jiang2013graph}    6) Non-negative Matrix Factorization \cite{lee1999learning}  7) Graph Regularized Non-negative Matrix Factorization (GNMF) \cite{cai2011graph}  8) Manifold Regularized Matrix Factorization (MMF) \cite{zhang2013low} 9) Robust PCA (RPCA) \cite{candes2011robust} 10) Fast Robust PCA on Graphs (A) 11)  Fast Robust PCA on Graphs (B) and 12)  Fast Robust PCA on Graphs (C). Two types of corruptions are introduced in the data: 1) Block occlusions and 2) Random missing values. NCut, LE, GLPCA, MMF and GNMF are evaluated with a good sample graph $G_1$. FRPCA(A) corresponds to our model evaluated with a good sample and a good feature graph, FRPCA(B) to a good sample graph and a noisy feature graph and FRPCA(C) to a noisy sample and feature graph. NMF and GNMF require non-negative data so they were not evaluated for the USPS and MFeat datasets because they are negative as well. MFeat is a non-image dataset so it is not evaluated with block occlusions. Due to the large size of the MNIST dataset, we use FLANN algorithm (strategy 2) to construct the graphs, therefore we get noisy graphs in the presence of corruptions. }
        \label{fig:errors}
    \end{figure*}


\subsubsection{Comparison with Matrix Factorization Models}
 Fig.~\ref{fig:errors}  presents the clustering error for various matrix factorization and our proposed model. NMF and GNMF are not evaluated for the USPS and MFeat datasets as they are not originally non-negative. It can be seen that our proposed model FRPCAG(A) with the two good quality graphs performs better than all the other models in most of the cases both in the presence and absence of data corruptions. Even FRPCAG(B) with a good sample graph $G_1$ and a noisy feature graph $G_2$ performs reasonably well. This shows that our model is quite robust to the quality of graph $G_2$. However, as expected FRPCAG(C) performs worse for ORL, CMU PIE, COIL20, YALE and USPS datasets as compared to other models evaluated with a good sample graph $G_1$. Finally, our model outperforms others in most of the cases for the interesting case of MNIST dataset where both graphs $G_1$ and $G_2$ are noisy for all models under consideration. It is worth mentioning here that even though the absolute errors are quite high for FRPCAG on the MNIST dataset, it performs relatively better than the other models. As PCA is mostly used as a feature extraction or a pre-processing step for a variety of machine learning algorithms, a better absolute classification performance can be obtained for these datasets by using FRPCAG as a pre-processing step for supervised algorithms as compared to other PCA models.

\begin{figure*}[htbp]
    \centering
        \centering
        \includegraphics[width=1.0\textwidth]{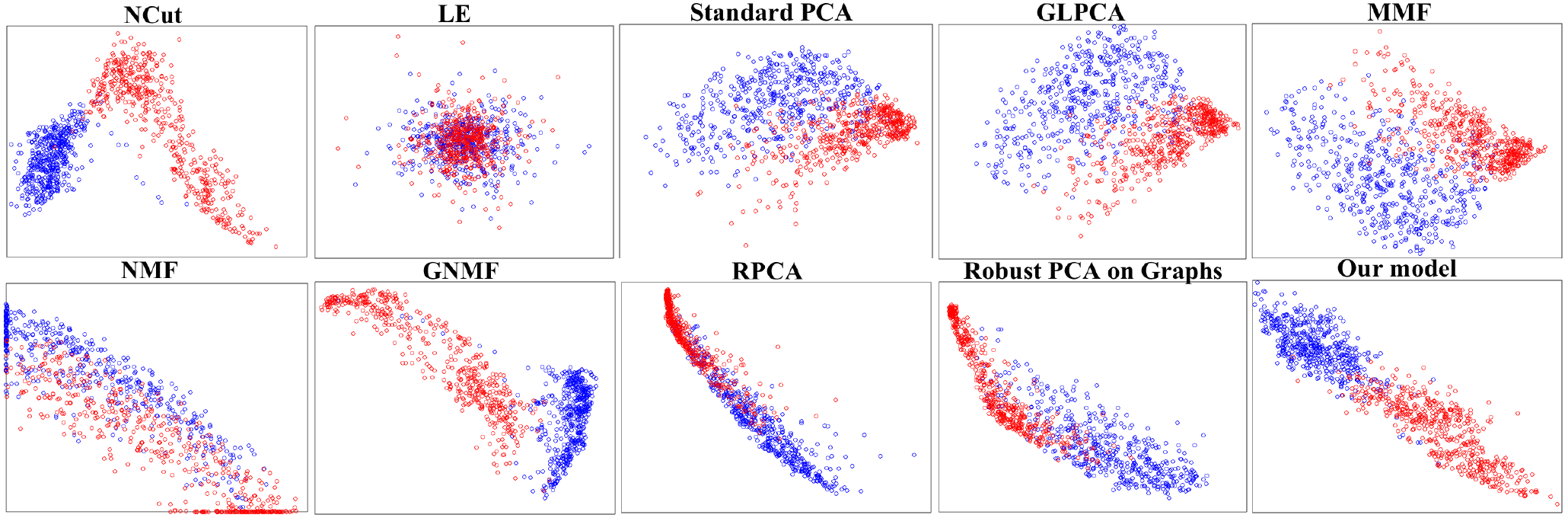}
         \caption{Principal Components of 1000 samples of digits 0 and 1 of the MNIST dataset in 2D space. For this experiment all the digits were corrupted randomly with 15\% missing pixels. Our proposed model (lower right) attains a good separation between the digits which is comparable and even better than other state-of-the-art dimensionality reduction models.}
        \label{fig:pcs}
    \end{figure*}
    
     \begin{figure*}[htbp]
    \centering
        \centering
        \includegraphics[width=1.0\textwidth]{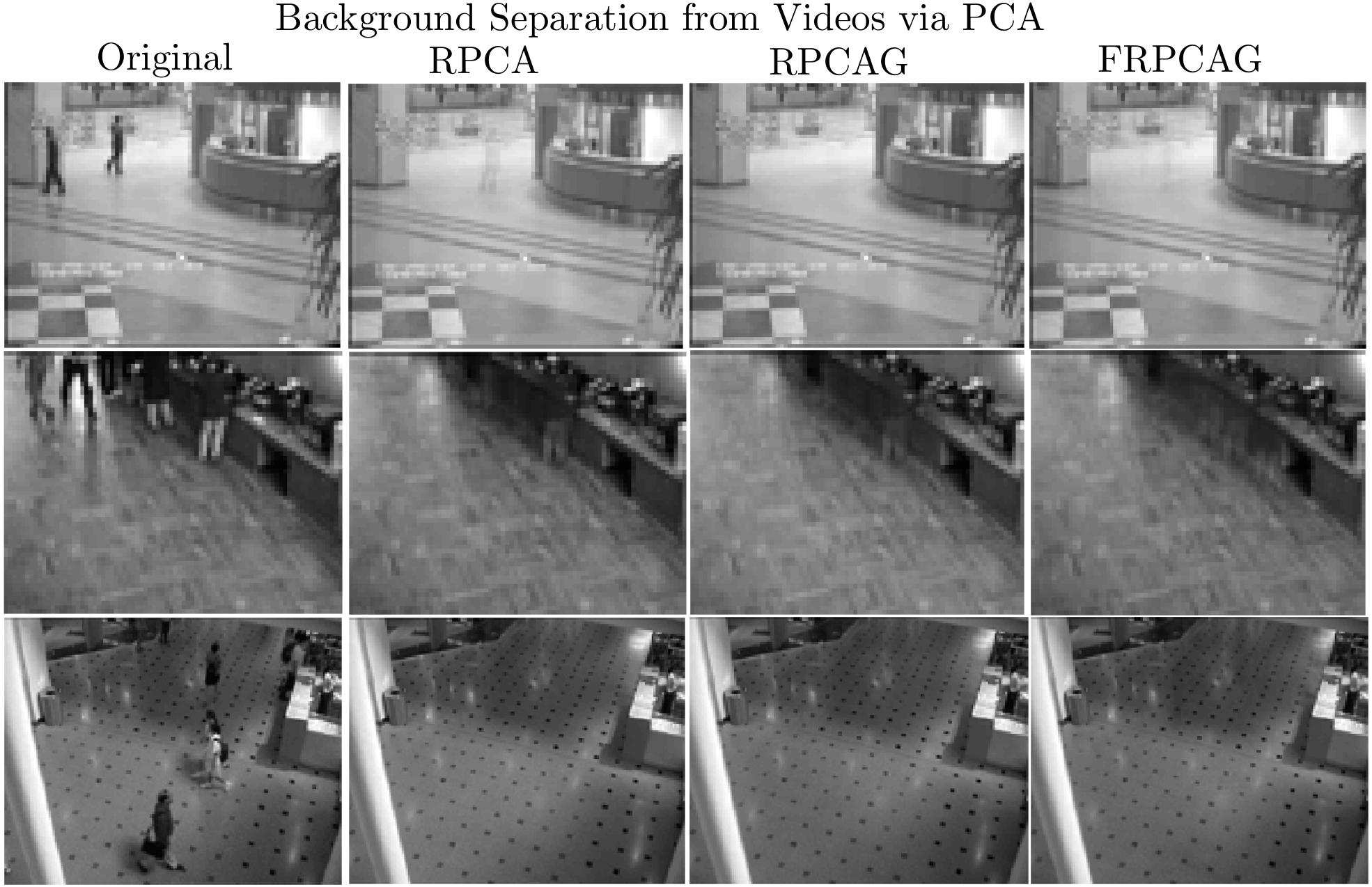}
         \caption{Static background separation from three videos. Each row shows the actual frame (left), recovered static low-rank background using RPCA, RPCAG and our proposed model. The first row corresponds to the video of a restaurant food counter, the second row to the shopping mall lobby and the third to an airport lobby. In all the three videos the moving people belong to the sparse component. Thus, our model is able to accurately separate the static portion from the three frames as good as the RPCAG. Our model converged in less than 2 minutes for each of the three videos, whereas RPCA and RPCAG converged in more than 45 minutes.}
        \label{fig:videos}
    \end{figure*}

\subsubsection{Comparison with Nuclear Norm based Models}
Fig.~\ref{fig:errors} also presents a comparison of the clustering error of our model with nuclear norm based models, i.e, RPCA and RPCAG for ORL dataset. This comparison is of specific interest because of the convexity of all the algorithms under consideration. As these models require an expensive SVD step on the whole low-rank matrix at every iteration of the algorithm, these experiments are performed on small ORL dataset. Clearly, our proposed model FRPCAG(A)  performs better than the nuclear norm based models even in the presence of large fraction of gross errors. Interestingly, even FRPCAG(B) with a noisy graph $G_2$ performs better than RPCAG with a good graph $G_1$. Furthermore, the performance of FRPCAG(C) with two noisy graphs is comparable to RPCAG with noisy graph, but still better than RPCA.

\subsection{Principal Components}
Fig.~\ref{fig:pcs} shows the principal components of $1000$ samples of MNIST dataset in two dimensional space obtained by various dimensionality reduction models. $500$ samples of digit $0$ and $1$ each are chosen and randomly corrupted by $15\%$ missing pixels for this experiment. Clearly, our proposed model attains a good separation between the digits 0 and 1 (represented by blue and red points respectively) comparable with other state-of-the-art dimensionality reduction models.

\subsection{Effect of the number of nearest neighbors for graphs}
In order to demonstrate the effect of number of nearest neighbors $K$ on the clustering performance of our model we perform a small experiment on the ORL dataset which has 400 images corresponding to 40 classes (10 images per class). We perform clustering for different values of $K = 5, 10, 25, 40$. The clustering errors are $17.5\%, 17\%, 23\%$ and $31\%$ respectively. Interestingly the minimum clustering error occurs for $K = 5, 10$ which is less or equal to the number of images per class. Thus, when the number of nearest neighbors $K$ is approximately equal to or less than the number of images per class then the images of the same class are more well connected and those across the classes have weak connections. This results in a lower clustering error. A good way to set $K$ is to use some prior information about the average number of samples per class or the rank of the dataset. For our experiments we use $K = 10$ for all the datasets and this value works quite well. The value of $K$ also depends on the number of data samples.  For big datasets, sparser graphs (obtained with lower values of $K$) tend to be more useful. For example, our experiments show that for the MNIST dataset (70,000 samples), $K=10$ is again a good value, even though the average number of samples per class is 7000. 
  
\subsection{Static background separation from videos}
In order to demonstrate the effectiveness of our model to recover low-rank static background from videos we perform experiments on 1000 frames of 3 videos available online. All the frames are vectorized and arranged in a matrix $X$ whose columns correspond to frames. The graph $G_{1}$ is constructed between the 1000 frames (columns of $X$) of the video and the graph $G_{2}$ is constructed between the pixels of the frames (rows of $X$) following the methodology of Section~\ref{sec:graphs}. Both graphs for all the videos are constructed without the prior knowledge of the mask of sparse errors (moving people). Fig.~\ref{fig:videos} shows the recovery of low-rank  frames for one actual frame of each of the videos. The leftmost plot in each row shows the actual frame, the other three show the recovered low-rank representations using RPCA, RPCAG and our proposed model (FRPCAG). The first row corresponds to a frame from the video of a restaurant food counter, the second row to the shopping mall lobby and the third row to an airport lobby. In each of the three plots it can be seen that our proposed model is able to separate the static backgrounds very accurately from the moving people which do not belong to the static ground truth.  Our model converged in less than 2 minutes for each of the three videos, whereas RPCA and RPCAG converged in more than 45 minutes. 


\subsection{Computational Time}
Table~\ref{tab:time} presents the computational time and number of iterations for the convergence of  FRPCAG, RPCAG and RPCA on different sizes and dimensions of the datasets. We also present the time needed for the graph construction. The computation is done on a single core machine with a 3.3 GHz processor without using any distributed or parallel computing tricks. An $\infty$ in the table indicates that the algorithm did not converge in 4 hours. It is notable that our model requires a very small number of iterations to converge irrespective of the size of the dataset. Furthermore, the model is orders of magnitude faster than RPCA and RPCAG. This is clearly observed from the experiments on MNIST dataset  where our proposed model is 100 times faster than RPCAG. Specially for MNIST dataset with 25000 samples, RPCAG and RPCA did not converge even in 4 hours whereas FRPCAG converged in less than a minute. 

\begin{table*}[htbp]
\caption{Computation times (in seconds) for graphs $G_{1}$, $G_{2}$, FRPCAG, RPCAG, RPCA and the number of iterations to converge for different datasets. The computation is done on a single core machine with a 3.3 GHz processor without using any distributed or parallel computing tricks. $\infty$ indicates that the algorithm did not converge in 4 hours. }
\centering
\resizebox{0.95\textwidth}{!}{\begin{tabular}[t]{| c | c | c | c | c | c | c | c | c | c | c | c |}
\hline
 \textbf{Dataset}  & \textbf{Samples}  & \textbf{Features}  & \textbf{Classes} & \multicolumn{2}{c |}{\textbf{Graphs}} & \multicolumn{2}{c | }{\textbf{FRPCAG}}  &  \multicolumn{2}{c | }{\textbf{RPCAG}} & \multicolumn{2}{c |}{\textbf{RPCA}} \\\cline{5-12}
 &     &      &      &  \textbf{$G_1$}  & \textbf{$G_2$}  & \textbf{time}  & \textbf{Iters}  & \textbf{time}  & \textbf{Iters} & \textbf{time}  & \textbf{Iters} \\\hline
 MNIST  &  5000    &   784    & 10 & 10.8   &  4.3  &   13.7   & 27  &  1345  & 325 & 1090   & 378 \\\hline
 MNIST & 15000   & 784       & 10    &   32.5   &  13.3   & 35.4 & 23 &  3801  & 412   &  3400 & 323 \\\hline
 MNIST & 25000   & 784       & 10 & 40.7   & 22.2   & 58.6   & 24  & $\infty$   &  $\infty$   &  $\infty$  & $\infty$ \\\hline
 ORL  & 300      & 10304   & 30 &  1.8   & 56.4    & 4.7 & 12 &  360  & 301   & 240  &  320  \\\hline
 USPS & 3500     &  256   & 10  & 5.8   & 10.8   & 1.76  & 16  & 900  &  410  &  790  & 350 \\\hline
 US census  & 2.5 million & 68   & -  & 540  & 42.3  & 3900 & 200  &  $\infty$ &  $\infty$ &  $\infty$ &  $\infty$ \\\hline
\end{tabular}}
\label{tab:time}
\end{table*}

 To demonstrate the scalability of our model for big datasets, we perform an experiment on the US census 1990 dataset available at the UCI machine learning repository. This dataset consists of approximately 2.5 million samples and 68 features. The approximate K-nearest neighbors graph construction strategy using the FLANN algorithm took only 540 secs to construct $G_{1}$ between 2.5 million samples and 42.3 secs. to construct $G_{2}$ between 68 features. We do not compare the performance of this model with other state-of-the-art models as the ground truth for this dataset is not available. However, we run our algorithm in order to see how long it takes to recover a low-rank representation for this dataset. It took 65 minutes and 200 iterations for the algorithm to converge on a single core machine with 3.3 GHz of CPU power.


 \vspace{-0.3cm}
  
\section{Conclusion}
We present Fast Robust PCA on Graphs (FRPCAG), a fast dimensionality reduction algorithm for mining clusters from high dimensional and large low-rank datasets. The idea lies on the novel concept of low-rank matrices on graphs. The power of the model lies in its ability to effectively exploit the hidden information about the intrinsic dimensionality of the smooth low-dimensional manifolds on which reside the clusterable signals and features of the data. Therefore, it targets an approximate recovery of low-rank signals by exploiting the local smoothness assumption of the samples and features of the data via graph structures only. In short our method leverages 1) smoothness of the samples on a sample graph and 2) smoothness of the features on a feature graph. The proposed method is convex, scalable and efficient and tends to outperform several other state-of-the-art exact low-rank recovery methods in clustering tasks that use the expensive nuclear norm. In an ordinary clustering task FRPCAG is approximately 100 times faster than nuclear norm based methods. The double graph structure also plays an important role towards the robustness of the model to gross corruptions. Furthermore, the singular values of the low-rank matrix obtained via FRPCAG closely approximate those obtained via nuclear norm based methods.

\section*{Acknowledgement}
The work of Nauman Shahid and Nathanael Perraudin is supported by the SNF grant no. 200021\_154350/1 for the project ``Towards signal processing on graphs''. The work of G. Puy is funded by the FP7 European Research Council Programme, PLEASE project, under grant ERC-StG-2011-277906. We would also like to thank Benjamin Ricaud for his valuable  suggestions to improve the paper.

\bibliographystyle{ieee}
\bibliography{pcabib}

\end{document}